\documentclass[runningheads,a4paper,orivec]{llncs}

\usepackage{times}
\usepackage{soul}
\usepackage{url}
\usepackage[hidelinks]{hyperref}
\usepackage[utf8]{inputenc}
\usepackage{caption}
\usepackage{booktabs}
\usepackage{amssymb}
\usepackage[normalem]{ulem}
\usepackage{MnSymbol}
\usepackage{stmaryrd}
\usepackage{graphicx}
\usepackage{amsmath}
\usepackage{algpseudocode}
\usepackage{algorithmicx,algorithm}
\newcommand{\decomp}[1]{\ensuremath{\operatorname{decomp}(#1)}}
\newcommand{\com}[2]{\ensuremath{\operatorname{com}(#1,#2)}}

\numberwithin{definition}{section}
\numberwithin{proposition}{section}
\numberwithin{theorem}{section}

\begin{document}

\mainmatter  % start of an individual contribution

\title{A Formal Framework for Reasoning about Agents' Independence in Self-organizing Multi-agent Systems}
\titlerunning{A Formal Framework for Reasoning about Agents' Independence in Self-organizing Multi-agent Systems}

\author{Jieting Luo\inst{1}, Beishui Liao\inst{1} \and
John-Jules Meyer\inst{2}}
\authorrunning{Jieting Luo et al.} % abbreviated author list (for running head)
\tocauthor{Jieting Luo, Beishui Liao and John-Jules Meyer}
\institute{Zhejiang University, Hangzhou, Zhejiang Province, China\\
\email{\{luojieting,baiseliao\}@zju.edu.cn}
\and
Utrecht University, Utrecht, the Netherlands\\
\email{J.J.C.Meyer@uu.nl}
}

\maketitle
\bibliographystyle{sigproc}

\begin{abstract}
Self-organization is a process where a stable pattern is formed by the cooperative behavior between parts of an initially disordered system without external control or influence. It has been introduced to multi-agent systems as an internal control process or mechanism to solve difficult problems spontaneously. However, because a self-organizing multi-agent system has autonomous agents and local interactions between them, it is difficult to predict the behavior of the system from the behavior of the local agents we design. This paper proposes a logic-based framework of self-organizing multi-agent systems, where agents interact with each other by following their prescribed local rules. The dependence relation between coalitions of agents regarding their contributions to the global behavior of the system is reasoned about from the structural and semantic perspectives. We show that the computational complexity of verifying such a self-organizing multi-agent system is in exponential time. We then combine our framework with graph theory to decompose a system into different coalitions located in different layers, which allows us to verify agents' full contributions more efficiently. The resulting information about agents' full contributions allows us to understand the complex link between local agent behavior and system level behavior in a self-organizing multi-agent system. Finally, we show how we can use our framework to model a constraint satisfaction problem.

\end{abstract}

\section{Introduction}
In the modern society artificial intelligence has been applied in many industries such as health care, retail and traffic. Since our nature presents beautiful ways of solving problems to us, biological insights have been the source of inspiration for the development of several techniques and methods to solve complex engineering problems. One of the examples is the adoption of self-organization from complex systems. Self-organization is a process where a stable pattern is formed by the cooperative behavior between parts of an initially disordered system without external control or influence. It has been introduced to multi-agent systems as an internal control process or mechanism to solve difficult problems spontaneously, especially if the system is operated in an open environment thereby having no perfect and a priori design to be guaranteed \cite{wang2002self}\cite{picard2005etto}\cite{valentini2014self}. One typical example using self organization mechanisms is ant colony optimization \cite{dorigo2006ant}, where ants collaborate to find out the optimal solution through laying down pheromone trails. In a wireless mobile sensor network, robots with sensors can deploy themselves to achieve optimal sensing coverage when the system designer is not aware of robots' interest or the operated environment \cite{khelil2016esa}. 

However, making a self-organizing system is highly challenging \cite{gorodetskii2012self}\cite{ye2016survey}. The traditional development of a multi-agent system is a top-down process which starts from the specification of the goal that the system needs to achieve to the development of specific agents. In this way, the goal of the system will guarantee to be achieved if the specific agents are implemented successfully. As the hierarchical in Figure.\ref{comparision}(left), the global objective can be achieved by modules E and F, and module E can be refined by agents A and B and module F can be refined by agents B, C and D. However, such an approach cannot be applied to the development of self-organizing multi-agent systems: because a self-organizing multi-agent system has autonomous agents and local interactions between them, the development of a self-organizing multi-agent system is usually a bottom-up process which starts from defining local components to examining global behavior. The distributive structure in Figure.\ref{comparision}(right) represents a self-organizing multi-agent system, where agents A, B, C and D interact with each other. Because of that, it is difficult to predict the behavior of the system from the system specification about autonomous agents and local interactions between them. Consequently, a self-organizing multi-agent system is usually evaluated through implementation. In other words, the complex link between local agent behavior and system level behavior in a self-organizing multi-agent system makes implementation the usual way of correctness evaluation. There have been some methodologies for developing self-organizing multi-agent systems (such as ADELFE \cite{bernon2003tools}\cite{bernon2002adelfe}), but they do not explain the complex link between local agent behavior and system level behavior of a self-organizing multi-agent system. The community of Self-Adaptive and Self-Organizing Systems (SASO) also highlights that it is still challenging to investigate how micro-level behavior lead to desirable macro-level outcomes. What if we can understand the complex link between micro-level agent behavior and macro-level system behavior in a self-organizing multi-agent system? If there is approach available that helps us understand how global system behavior emerges from agents' local interactions, the development of self-organizing multi-agent systems can be facilitated. For example, if we know a coalition of agents brings about a property independently and we want to change that property, what we have to do is to reconfigure the behavior of that coalition instead of the behavior of agents outside that coalition. As we can see, the independent relation between coalitions of agents in self-organizing multi-agent systems is an crucial issue to be investigated. 

\begin{figure}
  \centering
    \includegraphics[width=0.85\textwidth]{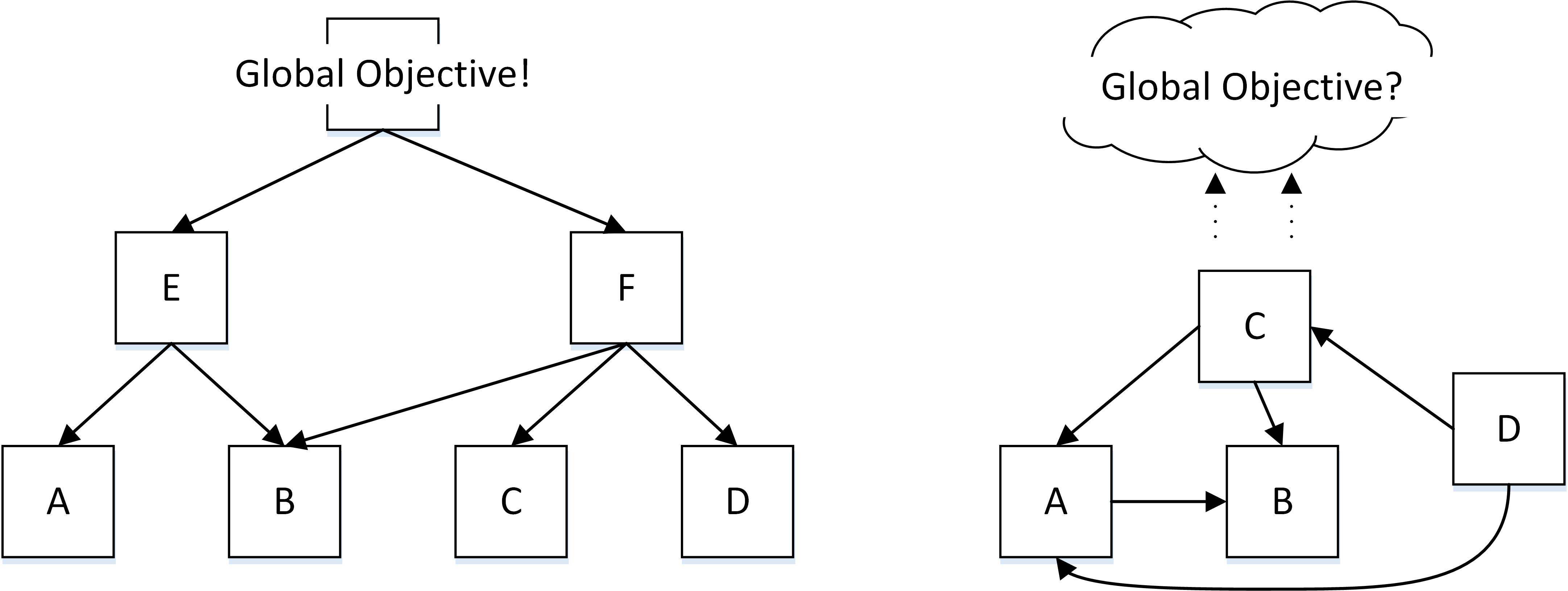}
    \caption{A comparison between top-down approach (left) and bottom-up approach (right).}\label{comparision}
\end{figure}

In theoretical computer science, logic has been used for proving the correctness of a system \cite{clarke2018model}\cite{clarke1986automatic}. Instead of implementing the system with respect to a specification, we can verify whether the system specification fulfills the global objective by checking logical formulas. That indeed provides a new way of evaluating a self-organizing multi-agent system apart from implementation. In this paper, we propose a logic-based framework of self-organizing multi-agent systems, where agents interact with each other by following their prescribed local rules. Based on the local rules, we define a structural property called independent component, a coalition of agents which do not get input from agents outside the coalition. Our semantics and the structure derived from communication between agents allow us not only to verify behavior of the system, but also to reason about the independence relation between coalitions of agents regarding their contributions to the global system behavior from two perspectives. Moreover, we propose a layered approach to decompose a self-organizing multi-agent system into different coalitions, which allows us to check agents' full contributions more efficiently. The resulting information about agents' full contributions allows us to understand the complex link between local agent behavior and system level behavior in a self-organizing multi-agent system. We finally show how we can use our framework to model a constraint satisfaction problem, where a solution based on self-organization is used.

The rest of the paper is organized as follows: Section 2 introduces the abstract framework to represent a self-organizing multi-agent system, proposing the notion of independent components; Section 3 provides the semantics of our framework to reason about agents' independence in terms of their contributions to the global behavior of the system; the model-checking problem is investigated in Section 4; Section 5 proposes a layered approach to decompose a self-organizing multi-agent system; in Section 6 we show how to use our framework to model constraint satisfaction problems; finally, related work and conclusion are provided in Sections 7 and 8 respectively.

\section{Abstract Framework}\label{abstractframework}
In this section, we will propose the model of this paper: a concurrent game structure extended with local rules and define the structural property of independent components whose behavior is independent on the behavior of the agents outside of the components.

\subsection{Self-organizing Multi-agent Systems}
The semantic structure of this paper is concurrent game structures (CGSs). It is basically a model where agents can simultaneously choose actions that collectively bring the system from the current state to a successor state. Compared to other kripke models of transaction systems, each transition in a CGS is labeled with collective actions and the agents who perform those actions. Moreover, we treat actions as first-class entities instead of using choices that are identified by their possible outcomes. Formally, 
\begin{definition}
A concurrent game structure is a tuple $\mathcal{S}=(k,Q,\pi,\Pi,ACT,d,\delta)$ such that:
\begin{itemize}
\item A natural number $k \geq 1$ of agents, and the set of all agents is $\Sigma = \{1,\ldots,k\}$; we use $A$ to denote a coalition of agents $A \subseteq \Sigma$;
\item A finite set $Q$ of states;
\item A finite set $\Pi$ of propositions;
\item A labeling function $\pi$ which maps each state $q \in Q$ to a subset of propositions which are true at $q$; thus, for each $q \in Q$ we have $\pi(q) \subseteq \Pi$;
\item A finite set $ACT$ of actions;
\item For each agent $i \in \Sigma$ and a state $q \in Q$, $d_i(q) \subseteq ACT$ is the non-empty set of actions available to agent $i$ in $q$; $D(q)=d_1(q) \times \ldots \times d_k(q)$ is the set of joint actions in $q$; given a state $q \in Q$, an action vector is a tuple $\langle \alpha_1, \ldots, \alpha_k \rangle$ such that $\alpha_i \in d_i(q)$;
\item A function $\delta$ which maps each state $q \in Q$ and a joint action $\langle \alpha_1, \ldots, \alpha_k \rangle \in D(q)$ to another state that results from state $q$ if each agent adopted the action in the action vector, thus for each $q \in Q$ and each $\langle \alpha_1, \ldots, \alpha_k \rangle \in D(q)$ we have $\delta (q, \langle \alpha_1, \ldots, \alpha_k \rangle)\in Q$.
\end{itemize}
\end{definition}

Note that the model is deterministic: the same update function adopted in the same state will always result in the same resulting state. A computation over $\mathcal{S}$ is an infinite sequence $\lambda = q_0,q_1,q_2,\ldots$ of states such that for all positions $i \geq 0$, there is a joint action $\langle \alpha_1, \ldots, \alpha_k \rangle \in D(q_i)$ such that $\delta(q_i, \langle \alpha_1, \ldots, \alpha_k \rangle)=q_{i+1}$. For a computation $\lambda$ and a position $i \geq 0$, we use $\lambda[i]$ to denote the $i$th state of $\lambda$.  More elaboration of concurrent game structures can be found in \cite{alur2002alternating}. 

Self-organization has been introduced into multi-agent systems for a long time to solve various problems in multi-agent systems \cite{ye2016survey}\cite{gorodetskii2012selfI}. It is a mechanism or a process which enables a system to finish a difficult task by the cooperative behavior between agents spontaneously \cite{di2005self}. In particular, agents in a self-organizing multi-agent system have local view of the system and the system reaches a desired state spontaneously without guided by any externals. In this paper, we argue that the cooperative behavior is guided by prescribed local rules that agents are supposed to follow with communication between agents as a prerequisite. Therefore, we can define a self-organizing multi-agent system as a concurrent game structure together with a set of local rules for agents to follow. For example, in ant colony optimization algorithms, ants are required to record their positions and the quality of their solutions (lay down pheromones) so that in later simulation iterations more ants locate better solutions. Before defining such a type of local rules, we first define what to communicate, which is given by an internal function.
\begin{definition}[Internal Functions]
Given a concurrent game structure $\mathcal{S}$, the internal function of an agent $i$ is a function $m_i: Q \to \mathcal{L}_{prop}$ that maps a state $q \in Q$ to a propositional formula over $\pi(q)$.
\end{definition}
The internal function returns the information that is provided by participating agents themselves at a given state and might be different from agent to agent. Depending on the application, we might have different interpretation on $m_i(q)$. For example, vehicles in a busy traffic situation are required to communicate their urgencies, and robots sensors in a self-deploy sensing network are required to communicate their sensing areas. Here we assume that agents are wise enough to process their local rules with the communicated information. A local rule is defined based on agents' communication as follows:
\begin{definition}[Abstract Local Rules]\label{alr}
Given a concurrent game structure $\mathcal{S}$, an abstract local rule for an agent $a$ is a tuple $\langle \tau_a, \gamma_a \rangle$ consisting of a function $\tau_a(q)$ that maps a state $q \in Q$ to a subset of agents, that is, $\tau_a(q) \subseteq \Sigma$, and a function $\gamma_a(M(q))$ that maps $M(q)=\{m_i(q) \mid i \in \tau_a(q)\}$ to an action available in state $q$ to agent $a$, that is, $\gamma_a(M(q)) \in d_a(q)$. We denote the set of all the abstract local rules as $\Gamma$ and a subset of abstract local rules as $\Gamma_A$ that are designed for coalition of agents $A$. 
\end{definition}
An abstract local rule consists of two parts: the first part $\tau_a(q)$ states the agents with whom agent $a$ is supposed to communicate in state $q$, and the second part $\gamma_a$ states the action that agent $a$ is supposed to take given the communication result with agents in $\tau_a(q)$ for their internals. Moreover, in \cite{alur2002alternating} and \cite{Knobbout2012Reasoning}, a rule (or a norm) is defined as a mapping $\gamma(q)$ that explicitly prescribes what agents need to do in a given state, which requires that system designers have complete information of the system including agents' internals such that the desired state as well as the legal computations can be identified. Differently, a local rule in this paper is defined based on agents' communication. Different participating agents might have different internal functions, making the communication results and thus the actions that are required to take different. Hence, the system allows agents to find out the desired state and how to get there in a self-organizing way. We see local rules not only as constraints but also guidance on agents' behavior, namely an agent does not know what to do if he does not communicate with other agents. Therefore, we exclude the case where agents get no constraint from their respective local rules. Notation $out$ denotes a set of computations and $out(q, \Gamma_A)$ is the set of computations starting from state $q$ where agents in coalition $A$ follow their respective local rules in $\Gamma_A$. A computation $\lambda=q_0,q_1,q_2,...$ is in $out(q_0, \Gamma_A)$ if and only if it holds that for all positions $i \geq 0$ there is a move vector $\langle \alpha_1, \ldots, \alpha_k \rangle \in D(\lambda[i])$ such that $\delta (\lambda[i], \langle \alpha_1, \ldots, \alpha_k \rangle)=\lambda[i+1]$ and for all $a \in A$ it is the case that $\alpha_a = \gamma_a(M(q))$. Because function $\gamma_a$ returns only an action that agent $a$ is allowed to take in any state, there will be only one computation in the set of computations $out(q, \Gamma_{\Sigma})$, which will be denoted as $\lambda^*(q)$ without the curly brackets outside in the rest of the paper. Now we are ready to define a self-organizing multi-agent system. Formally,
\begin{definition}[Self-organizing Multi-agent Systems]
A self-organizing multi-agent system (SOMAS) is a tuple $(\mathcal{S},\Gamma)$, where $\mathcal{S}$ is a concurrent game structure and $\Gamma$ is a set of local rules for agents in the system to follow.
\end{definition}
We will use the example in \cite{wooldridge2005obligations} for better understanding the above definitions.
\begin{example}
We consider a CGS scenario as Fig.\ref{two_trains} where there are two trains, each controlled by an agent, going through a tunnel from the opposite side. The tunnel has only room for one train, and the agents can either wait or go. Starting from state $q_0$, if the agents choose to go simultaneously, the trains will crash, which is state $q_4$; if one agent goes and the other waits, they can both successfully go through the tunnel without crashing, which is $q_3$. 
\begin{figure}[h]
  \centering
    \includegraphics[width=0.7\textwidth]{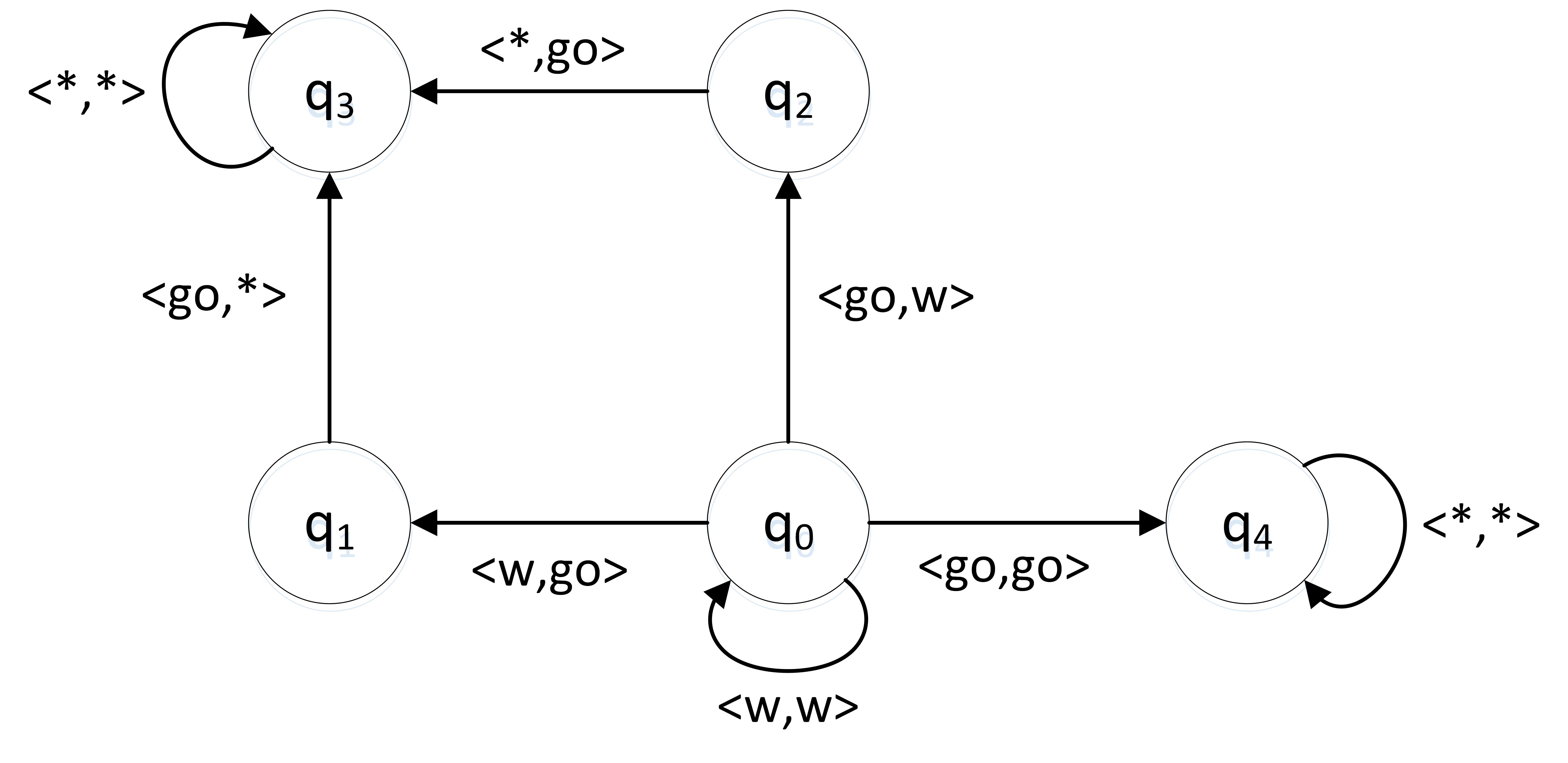}
    \caption{A CGS example.}\label{two_trains}
\end{figure}

Local rules $\langle \tau_1, \gamma_1 \rangle$ and $\langle \tau_2, \gamma_2 \rangle$ are prescribed for the agents to follow: both agents communicate with each other for their urgency $u_1$ and $u_2$ in state $q_0$; the one who is more urgent can go through the tunnel first and the other one has to wait for it; after $q_0$ the agent who waits can go. We formalize as follows. In state $q_0$, $\tau_1(q_0)=\{a_1, a_2\}$ and $\tau_2(q_0)=\{a_1, a_2\}$,
\[ \gamma_1(u_1,u_2)= \left\{ \begin{array}{ll}
         go & \mbox{if $a_1$ is more urgent than or}\\
		 & \mbox{as urgent as $a_2$};\\
         wait & \mbox{otherwise}.\end{array} \right. \] 
\[ \gamma_2(u_1,u_2)= \left\{ \begin{array}{ll}
         go & \mbox{if $a_1$ is less urgent than $a_2$};\\
         wait & \mbox{otherwise}.\end{array} \right. \] 
In state $q_1$ and $q_2$,  $\tau_1(q_1)=\{a_1\}$ and $\tau_2(q_2)=\{a_2\}$, and
\[ \gamma_1(u_1)= go, \quad \gamma_2(u_2)= go. \] 
Given the above local rules, if $a_1$ is more urgent than or as urgent as $a_2$ w.r.t. $u_1$ and $u_2$, the desired state $q_3$ is reached along with computation $q_0,q_2,q_3 \ldots$; if $a_1$ is less urgent than $a_2$ w.r.t. $u_1$ and $u_2$, the desired state $q_3$ is reached along with computation $q_0,q_1,q_3 \ldots$. As we can see, any legal computation is not prescribed by system designers, because it depends on the urgencies that are provided by agents themselves. Instead, the agents can find out how to get to the desire state $q_3$ by themselves through following their local rules. Certainly, agents can collaborate to cross the tunnel successfully without communication, but that requires an external who is aware of the available actions for each train and the game structure to make a plan for them, which is not allowed in a self-organizing multi-agent system. In a self-organizing multi-agent system, each train does not know the available actions from the other train but it can follow its local rule to behave based on the communication result instead of listening to an external to behave.
\end{example}

\subsection{Full Contribution}
Similar to ATL \cite{alur2002alternating}, our language ATL-$\Gamma$ is interpreted over a concurrent game structure $\mathcal{S}$ that has the same propositions and agents. It is an extension of classical propositional logic with temporal cooperation modalities and path quantifiers. A formula of the form $\langle A \rangle \psi$ means that coalition of agents $A$ will bring about the subformula $\psi$ by following their respective local rules in $\Gamma_A$, no matter what agents in $\Sigma \backslash A$ do, where $\psi$ is a temporal formula of the form $\bigcirc \varphi$, $\Box \varphi$ or $\varphi_1 \mathcal{U}\varphi_2$ (where $\varphi$, $\varphi_1$, $\varphi_2$ are again formulas in our language). Formally, the grammar of our language is defined below, where $p \in \Pi$ and $A \subseteq \Sigma$:
\[ \varphi :: = p \mid \lnot \varphi \mid \varphi_1 \land \varphi_2 \mid \langle A \rangle \bigcirc \varphi \mid \langle A \rangle \Box \varphi \mid \langle A \rangle \varphi_1 \mathcal{U}\varphi_2 \hspace{5 mm} \]

Given a self-organizing multi-agent system $(\mathcal{S},\Gamma)$, where $\mathcal{S}$ is a concurrent game structure and $\Gamma$ is a set of local rules, and a state $q \in Q$, we define the semantics with respect to the satisfaction relation $\models$ inductively as follows:
\begin{itemize}
\item $\mathcal{S}, \Gamma, q \models p$ iff $p \in \pi(q)$;
\item $\mathcal{S}, \Gamma, q \models \lnot \varphi$ iff $\mathcal{S}, \Gamma, q \not \models \varphi$;
\item $\mathcal{S}, \Gamma, q \models \varphi_1 \land \varphi_2$ iff $\mathcal{S}, \Gamma, q \models \varphi_1$ and $\mathcal{S}, \Gamma, q \models \varphi_2$;
\item $\mathcal{S}, \Gamma, q \models \langle A \rangle \bigcirc \varphi$ iff for all $\lambda \in out(q,\Gamma_A)$, we have $\mathcal{S}, \Gamma,\lambda[1]\models \varphi$;
\item $\mathcal{S}, \Gamma, q \models \langle A \rangle \Box \varphi$ iff for all $\lambda \in out(q,\Gamma_A)$ and all positions $i \geq 0$ it holds that $\mathcal{S}, \Gamma,\lambda[i] \models \varphi$;
\item $\mathcal{S}, \Gamma, q \models \langle A \rangle \varphi_1 \mathcal{U}\varphi_2$ iff for all $\lambda \in out(q,\Gamma_A)$ there exists a position $i \geq 0$ such that for all positions $0 \leq j \leq i$ it holds that $\mathcal{S}, \Gamma,\lambda[j] \models \varphi_1$ and $\mathcal{S}, \Gamma,\lambda[i] \models \varphi_2$.
\end{itemize}
Dually, we write $\langle A \rangle \Diamond \varphi$ for $\langle A \rangle \top \mathcal{U} \varphi$. Importantly, when we say a coalition of agents ensures a temporal formula by following their respective local rules, it means that agents in the coalition ensure a temporal formula if they take the actions that their local rules return, no matter whether agents outside of the coalition take the actions that are allowed to take or not. It is merely interpreted from our semantics, while agents' dependence relation in terms of communication does not play a role.
\begin{proposition}\label{bigger}
Given an SOMAS $(\mathcal{S},\Gamma)$ and a coalition $A$, for any coalition $A^\prime \supseteq A$, it holds that 
\[\mathcal{S}, \Gamma, q \models \langle A \rangle \psi \Rightarrow \mathcal{S}, \Gamma, q \models \langle A^\prime \rangle \psi,\]
\end{proposition}
\begin{proof}
Because $A \subseteq A^\prime$, computation set $out(q,\Gamma_{A^\prime}) \subseteq out(q,\Gamma_A)$. Thus, if $\mathcal{S}, \Gamma, q \models \langle A \rangle \psi$, meaning that $\psi$ holds in all $\lambda \in out(q,\Gamma_A)$, then $\psi$ will also hold in all $\lambda \in out(q,\Gamma_{A^\prime})$. Therefore, $\mathcal{S}, \Gamma, q \models \langle A^\prime \rangle \psi$.
\end{proof}
It means if a coalition ensures an temporal formula, then any coalition that contains that coalition will also ensure that temporal formula. Notice that $A^\prime$ can be the whole agent set $\Sigma$. 

\begin{example}
According to the local rules in the two-train example, the train who is more urgent can go through the tunnel first and the other one has to wait for him. We have that one train by itself cannot bring about the result of passing through the tunnel without crash through following the local rule, which can be expressed:
\[\mathcal{S}, \Gamma, q_0 \not \models \langle a_1 \rangle \Diamond \text{passed},\]
\[\mathcal{S}, \Gamma, q_0 \not \models \langle a_2 \rangle \Diamond \text{passed}.\]
Instead, both trains have to cooperate to bring about the result. Thus, we have that both agents by themselves can bring about the result of passing through the tunnel without crash through following the local rules, which can be expressed:
\[\mathcal{S}, \Gamma, q_0 \models \langle a_1, a_2 \rangle \Diamond \text{passed}.\]
\end{example}

As we mentioned in the introduction, because a self-organizing multi-agent system has autonomous agents and local interactions between them, we are not clear about how the local components lead to the global behavior of the system, making a self-organizing multi-agent system usually evaluated through implementation. One possible solution to understand the complex link between local agent behavior and system level behavior is to divide the system into components, each of which has contribution to the global behavior of the system and is independent on agents outside the coalition. The principle is that, when studying the behavior of one component, we do not need to think about the influences coming from agents outside the component. The independence between different components in terms of their contributions to the global behavior of the system allows us to understand the complex link. With this idea, we first define a notion of \textit{independent components} over a self-organizing multi-agent system. 
\begin{definition}[Independent Components] \label{independent}
Given an SOMAS $(\mathcal{S},\Gamma)$, a coalition of agents $A$ is an independent component w.r.t. a state $q$ iff for all $a \in A$ and its abstract local rule $\langle \tau_a, \gamma_a \rangle$ it is the case that $\tau_a(q) \subseteq A$; a coalition of agents $A$ is an independent component w.r.t. a set of computations $out$ iff for all $\lambda \in out$ and $q \in \lambda$ $A$ is an independent component w.r.t. state $q$.
\end{definition}
An independent component w.r.t. state $q$ is a coalition of agents $A$ which only gets input information from agents inside coalition $A$ in state $q$. In other words, an independent component might output information to agents in $\Sigma \backslash A$, but never get input from agents in $\Sigma \backslash A$.
\begin{proposition}
Given an SOMAS $(\mathcal{S},\Gamma)$ and two coalitions $A$ and $B$ where $A \cap B \not= \emptyset$, if both $A$ and $B$ are independent components, then $A \cap B$ is also an independent component.
\end{proposition}
\begin{proof}
Because $A$ is an independent component and $A \cap B \subseteq A$, agents in $A \cap B$ do not get input from agents in $\Sigma \backslash A$. For the same reason, agents in $A \cap B$ do not get input from agents in $\Sigma \backslash B$. Thus, agents in $A \cap B$ do not get input from agents in $(\Sigma \backslash A) \cup (\Sigma \backslash B) = \Sigma \backslash (A \cap B)$. Therefore, $A \cap B$ is also an independent component.
\end{proof}
 
\begin{example}
In the two-train example, as in state $q_0$ both trains need to communicate with each other as their local rules require, neither $a_1$ nor $a_2$ is an independent component w.r.t. state $q_0$. Because the system only consists of two trains, the two trains form an independent component w.r.t. state $q_0$. Because in states $q_1$ and $q_2$ each train only gets urgency from itself to go ahead, $a_1$ is an independent component w.r.t. state $q_1$ and $a_2$ is an independent component w.r.t. state $q_2$.
\end{example}

Using the notion of independent components and our language, We then propose the notion of \emph{semantic independence, structural independence and full contribution} to characterize the independence between agents from different perspectives.
\begin{definition}[Semantic Independence, Structural Independence and Full Contribution]\label{fullcontribution}
Given an SOMAS $(\mathcal{S},\Gamma)$, a coalition of agents $A$ and a state $q$,
\begin{itemize}
\item $A$ is semantically independent with respect to a temporal formula $\psi$ from $q$ iff $\mathcal{S}, M, \Gamma, q \models \langle A \rangle \psi$;
\item $A$ is structurally independent from $q$ iff $A$ is an independent component w.r.t. the set of computations $out(q,\Gamma_A)$;
\item $A$ has full contribution to $\psi$ in $q$ iff $A$ is the minimal (w.r.t. set-inclusion) coalition that is both semantically independent with respect to $\psi$ and structurally independent from $q$.
\end{itemize}
\end{definition}
The notion of full contribution captures the property of coalition $A$ in terms of independence from two different perspectives: semantically, coalition $A$ ensures $\psi$ through following its local rules no matter what other agents do; structurally, coalition $A$ do not communicate with other agents when following its local rules no matter what other agents do. Notice that when we say coalition $A$ has full contribution to $\psi$ in $q$, it is important for coalition $A$ to be the minimal coalition that is both semantically independent with respect to $\psi$ and structurally independent from state $q$. This is because there might exist multiple coalitions that are both semantically independent with respect to $\psi$ and structurally independent from state $q$ and the set of all agents $\Sigma$ is apparently one of them. In other words, coalition $A$ has full contribution to $\psi$ because any subset of coalition $A$ is either semantically dependent with respect to $\psi$ or structurally dependent. The following example illustrates why we need both semantic independence and structural independence to characterize the full contribution of a coalition of agents.

\begin{example}
Consider the transition system in Fig.\ref{SS}(top). $\langle \alpha, * \rangle$ is interpreted as an action vector where agent $a_1$ performs action $\alpha$ and agent $a_2$ does whatever he can. Local rules are prescribed for agents $a_1$ and $a_2$: in state $q_0$ $a_1$ needs to communicate with $a_2$ for some valuable information and is supposed to do $\alpha$ based on the communication result. As we can see from the structure, $a_1$ brings about $p$ no matter what $a_2$ does, which means that coalition $\{a_1\}$ is semantically independent. However, since $a_1$ needs to communicate with $a_2$ in state $q_0$, coalition $\{a_1\}$ is not structurally independent.

Consider the transition system in Fig.\ref{SS}(down). $\langle \alpha, * \backslash \alpha \rangle$ is interpreted as an action vector where agent $a_1$ performs action $\alpha$ and agent $a_2$ deviates from action $\alpha$. Similar for $\langle * \backslash \alpha, \alpha \rangle$. Local rules are prescribed for agents $a_1$ and $a_2$: in state $q_0$, both agents do not need to communicate with each other and are supposed to do $\alpha$. As we can see from the structure, $a_1$ and $a_2$ can bring about $p$ through following their local rules but neither $a_1$ nor $a_2$ can achieve that by itself, which means that coalitions $\{a_1\}$ or $\{a_2\}$ is not semantically independent. But since they do not need to communicate with each other, it is structurally independent.

\begin{figure}[h]
  \centering
    \includegraphics[width=0.5\textwidth]{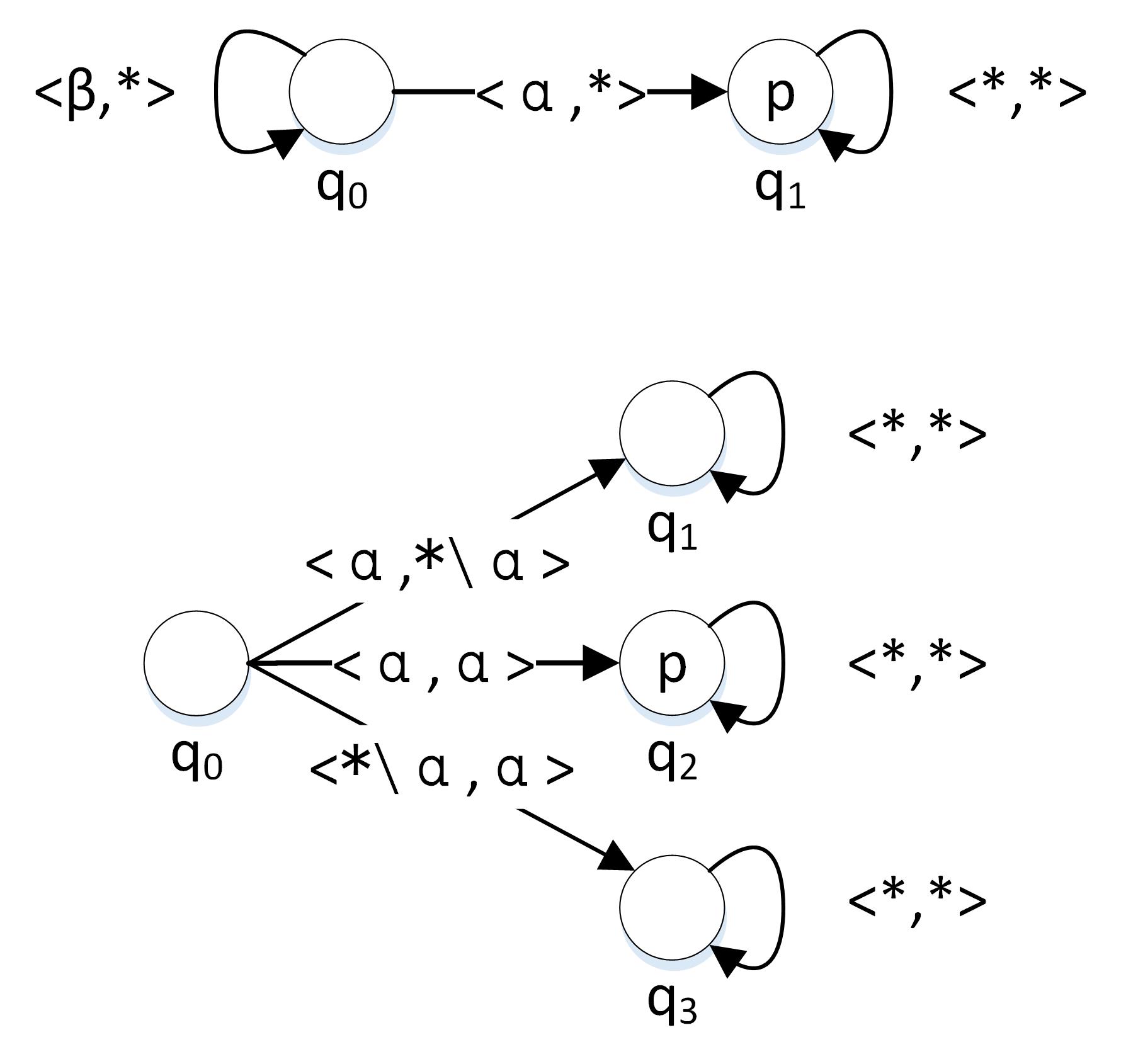}
    \caption{Comparison between semantic independence (top) and structural independence (down).}\label{SS}
\end{figure}

As for the two-train example, neither train, namely $a_1$ or $a_2$ as a coalition, has full contribution to the result of passing through the tunnel without crash. The reasons are listed as follows: any single train cannot ensure the result of passing through the tunnel without crash, which means that any single train is not semantically independent. Moreover, both trains follow their local rules to communicate with each other in state $q_0$, which means that any single train is not an independent component w.r.t. state $q_0$ thus not being structurally independent.

The two trains have full contribution to the result of passing through the tunnel without crash, because both agents by themselves can bring about the result of passing through the tunnel without crash through following the local rules, and the coalition of two trains is obviously an independent component w.r.t. $out(q_0, \Gamma_{\{a_1,a_2\}})$, which means that it is structurally independent, and the coalition of two trains is obviously the minimal coalition that is both semantically independent w.r.t. the result of no crash and structurally independent from state $q_0$.
\end{example}

\begin{proposition}
Given an SOMAS $(\mathcal{S},M,\Gamma)$, a state $q$ and a temporal formula $\psi$, there does not exist two different coalitions $A$ and $B$ such that $A \cap B \not= \emptyset$ and both $A$ and $B$ has full contribution to $\psi$ in $q$.
\end{proposition}
\begin{proof}
Suppose there exists two different coalitions $A$ and $B$ that have full contribution to $\psi$, which means that both $A$ and $B$ are the minimal (w.r.t. set-inclusion) coalitions that are both semantically independent with respect to $\psi$ and structurally independent from $q$. Because $\mathcal{S}, M, \Gamma, q \models \langle A \rangle \psi$ and $\mathcal{S}, M, \Gamma, q \models \langle B \rangle \psi$, we have $\mathcal{S}, M, \Gamma, q \models \langle A \cap B \rangle \psi$. Because $A \cap B \not= \emptyset$ and both $A$ and $B$ are independent components, we have $A \cap B$ is also an independent component. Therefore, $A \cap B$ is both semantically independent with respect to $\psi$ and structurally independent from $q$. If $A \cap B \subset A $, $A$ is not the minimal (w.r.t. set-inclusion) coalition that is both semantically independent with respect to $\psi$ and structurally independent from $q$. If $A \cap B = A $, which means that $A \subset B$, then $A$, not $B$, is the minimal (w.r.t. set-inclusion) coalition that is both semantically independent with respect to $\psi$ and structurally independent from $q$. Contradiction!
\end{proof}
The above proposition is consistent with an intuition: when semantically coalition $A$ ensures $\psi$ through following local rules but $A$ is not structurally independent from $q$, meaning that agents in $A$ need to communicate with agents outside $A$ to ensure $\psi$, we will enlarge the coalition through including the agents with which agents from $A$ communicate. The resulting coalition is the minimal (w.r.t. set-inclusion) coalition that is both semantically independent with respect to $\psi$ and structurally independent from $q$, which is also unique.

\section{Model Checking}\label{sec:verification}
Our logic-based framework provides us another approach to verify a self-organizing multi-agent system. If we only care about whether the system will bring about a property, we only need to check whether a formula in ATL-$\Gamma$ with the whole set of agents is satisfied in a certain state; if we want to know whether a coalition of agents will bring about a property independently, we need to check whether a coalition of agents has full contribution to a temporal formula in a certain state. In this section, we will investigate how difficult to answer these two model-checking problems. We first measure the model-checking problem for ATL-$\Gamma$. In order to answer that, we extend our concurrent game structure $\mathcal{S}$ by adding new propositions to states that indicate for each local rule $\langle \tau_a, \gamma_a \rangle \in \Gamma$, whether or not it is followed by a corresponding agent. For this purpose, we define the extended game structure $\mathcal{S}^F=(k,Q^F,\pi^F,\Pi^F,ACT,d^F,\delta^F)$ as follows:
\begin{itemize}
\item $Q^F=\{\langle\bot,q\rangle \mid q\in Q\} \cup \{\langle q^\prime, q\rangle \mid q^\prime, q \in Q \text{ and } q \text{ is a successor of } q^\prime \text{ in } \mathcal{S}\}$. In other words, a state of the form $\langle\bot, q\rangle$ of $\mathcal{S}^F$ corresponds to the game structure $\mathcal{S}$ being in state $q$ at the beginning of a computation, and a state of the form $\langle q^\prime, q\rangle$ corresponds to $\mathcal{S}$ being in state $q$ during a computation whose previous state was $q^\prime$.
\item For each agent $a \in \Sigma$, there is a new proposition \textit{followed}; that is $\Pi^F=\Pi \cup \{followed_a \mid a \in \Sigma\}$.
\item For each state of the form $\langle\bot,q\rangle \in Q^F$, we have $\pi^F(\langle\bot, q\rangle)=\pi(q)$; For each state $\langle q^\prime, q\rangle\in Q^F$, we have 
\begin{align*}
\pi^F(\langle q^\prime,q \rangle)=&\pi(q)\cup \{followed_a \mid \text{ there is a move vector } \langle \alpha_1, \ldots, \alpha_k \rangle \in D(q^\prime) \\
&\text{such that } \delta(q^\prime, \langle \alpha_1, \ldots, \alpha_k \rangle)=q \text{ and } \alpha_a \in \gamma_a(M(q)).\}
\end{align*}
\item For each player $\alpha \in \Sigma$ and each state $\langle \cdot, q \rangle \in Q^F$, we have $d^F_{\alpha}(\langle \cdot, q \rangle) = d_a(q)$;
\item For each state $\langle \cdot, q \rangle \in Q^F$ and each move vector $\langle \alpha_1, \ldots, \alpha_k \rangle \in D(q)$, we have $\delta^F(\langle \cdot, q \rangle, \langle \alpha_1, \ldots, \alpha_k \rangle) = \delta(q,\langle \alpha_1, \ldots, \alpha_k \rangle)$.
\end{itemize}
That is how we transfer $\mathcal{S}$ to $\mathcal{S}^F$, and one computation in $\mathcal{S}$ corresponds to one computation in $\mathcal{S}^F$. The new propositions $\{followed_a \mid a \in \Sigma\}$ allow us to identify the computations that follow the local rules. Using $\mathcal{S}^F$, we encode agents' following of local rules as propositions in the states. Therefore, evaluating formulas of the form $\langle A \rangle \psi$ over states of $\mathcal{S}$ can be reduced to evaluating standard ATL formulas over states of $\mathcal{S}^F$. In classic ATL \cite{alur2002alternating}, a formula of the form $\llangle A \rrangle \psi$ means there exists a set $F_A$ of strategies, one for each player in $A$, to bring about $\psi$, which can be interpreted as agents' capacity of bringing about a property and is different from the semantics of $\langle A \rangle \psi$ in this paper. Namely, given a SOMAS $(\mathcal{S},\Gamma)$, a state $q$ and a set of agents $A$, a formula $\mathcal{S}, \Gamma, q \models \langle A \rangle \psi$ holds iff the state $\langle \bot, q \rangle$ of the extended game structure $\mathcal{S}^F$ satisfies the following ATL$^*$ formula:
\[ \llangle A \rrangle (\bigwedge_{a \in A} \Box followed_a \land \psi).\]
As we can see, even though coalition of agents $A$ has the capacity to bring about $\psi$, formula $\mathcal{S}, \Gamma, q \models \langle A \rangle \psi$ does not necessarily hold, because coalition $A$ has to follow its local rules to achieve that. To see why, consider the two-train example. 
\begin{example}
Suppose we change the local rules to be the following:
\[ \gamma_1(u_1,u_2)= \left\{ \begin{array}{ll}
         go & \mbox{if $a_1$ is more urgent than $a_2$};\\
         wait & \mbox{otherwise}.\end{array} \right. \] 
\[ \gamma_2(u_1,u_2)= \left\{ \begin{array}{ll}
         go & \mbox{if $a_1$ is less urgent than $a_2$};\\
         wait & \mbox{otherwise}.\end{array} \right. \]
When both trains have the same urgency, it will result in deadlock instead of passing through the tunnel if both trains follow their local rules to wait, which can be expressed in
\[\mathcal{S}, \Gamma, q_0 \not \models \langle a_1, a_2 \rangle \Diamond \text{passed}.\]
\[\mathcal{S}, \Gamma, q_0 \models \langle a_1, a_2 \rangle \Diamond \text{deadlock}.\]
However, it is clear that both trains can cooperate to pass through the tunnel without crash, which can be expressed in
\[\mathcal{S}, q_0 \models \llangle a_1, a_2 \rrangle \Diamond \text{passed}.\]
\end{example}

Although checking an ATL-${\Gamma}$ formula can be reduced to checking a logically equivalent ATL$^*$ formula, it still can be done in an efficient way so that the corresponding complexity bounds are much lower than those for general ATL$^*$ model checking.

\begin{proposition}\label{verification}
The model-checking problem for ATL-$\Gamma$ can be solved in time $\mathcal{O}(m^2 \cdot n \cdot l)$ for a self-organizing multi-agent system with $m$ states, a coalition of agents $|A|=n$ and a formula of length $l$.
\end{proposition}
\begin{proof}
We adopt the proof strategy from \cite{alur2002alternating}. We first construct the extended game structure $S^F$, where the obedience/violation of local rules are encoded in the structure. To verify whether a temporal formula $\psi$ can be enforced by a coalition of agents $A$ through following its local rules in a self-organizing multi-agent system at state $q$, we need to check formula $\mathcal{S}, \Gamma, q \models \langle A \rangle \psi$, which can be reduced to evaluating an ATL$^*$ formula for $\mathcal{S}^F$ and $\langle \bot, q \rangle$:
\[ \mathcal{S}^F, \langle \bot, q \rangle \models \llangle A \rrangle (\bigwedge_{a \in A} \Box followed_a \land \psi).\]
We then construct a 2-player turn-based synchronous game $\mathcal{S}^F_A$, where player 1 controls all the actions of coalition $A$ (called A-move) leading to an auxiliary state, after which player 2 controls all the actions of coalition $\Sigma \backslash A$ (called B-move) leading to the next state of the original transition. Such a game can be interpreted as an AND-OR graph, for which we can solve certain invariance and reachability problems in linear time. Because following the local rules is a necessary condition for coalition $A$ to win the game, we can remove all the outgoing A-move in $\mathcal{S}^F$ from an state where $\lnot \bigwedge_{a \in A} followed_a$ holds and its outgoing transitions, which can be done in polynomial time to the number of states and the number of agents in $A$. If the original game structure $\mathcal{S}$ has $m$ states, then the turn-based synchronous structure $\mathcal{S}^F_A$ has $\mathcal{O}(m^2)$ states. We then perform normal model checking for $\llangle A \rrangle \psi$ in $\mathcal{S}^F_A$, which can be done in polynomial time to the length of $\psi$. Therefore, checking formula $\mathcal{S}, \Gamma, q \models \langle A \rangle \psi$ is in the complexity of $\mathcal{O}(m^2 \cdot n \cdot l)$.
\end{proof}

We then measure the complexity of verifying a self-organizing multi-agent system, namely checking whether a coalition of agents has full contribution to a temporal formula in a certain state. A coalition has full contribution to a temporal formula in a state if only if it is the minimal (w.r.t. set-inclusion) coalition that is both semantically independent with respect to that temporal formula and structurally independent from that state. To verify structural independence, we can encode agents' communication in the structure like what we did for the obedience/violation of local rules. Namely, given a coalition $A$, we define the extended game structure $\mathcal{S}^E=(k,Q,\pi^E,\Pi^E,ACT,d,\delta)$ as follows:
\begin{itemize}
\item For each agent $a \in A$, there is a new proposition $InA$; that is $\Pi^E=\Pi \cup \{InA_a \mid a \in A\}$.
\item $\pi^E(q)=\pi(q) \cup \{InA_a \mid \tau_a(q) \subseteq A\}$.
\end{itemize}
The new propositions $\{InA_a \mid a \in A\}$ allow us to identify the states where agents in $A$ only get input from agents inside the coalition. We then construct the extended game structure $\mathcal{S}^{EF}$, where the obedience/violation of local rules are encoded in the structure. Therefore, evaluating whether a coalition of agents $A$ is an independent component with respect to $out(q,\Gamma_A)$ can be again reduced to standard ATL formulas over states of $\mathcal{S}^{EF}$. 
\begin{proposition}\label{verification2}
The model-checking problem for structural independence can be solved in time $\mathcal{O}(m^2 \cdot n^2)$ for a self-organizing multi-agent system with $m$ states and a coalition of agents $|A|=n$.
\end{proposition}
\begin{proof}
Given a SOMAS $(\mathcal{S},\Gamma)$, a state $q$ and a set of agents $A$, $A$ is an independent component with respect to $out(q,\Gamma_A)$ iff the state $\langle \bot, q \rangle$ of the extended game structure $\mathcal{S}^{EF}$ satisfies the following ATL$^*$ formula:
\[ \mathcal{S}^{EF}, \langle \bot, q \rangle \models \llangle A \rrangle \bigwedge_{a \in A} \Box (followed_a \land InA_a)\]
As we have to check $InA_a$ for each agent in $A$, checking the above formula is in the complexity of $\mathcal{O}(m^2 \cdot n^2)$.
\end{proof}
With the results of Propositions \ref{verification} and \ref{verification2}, we can measure the complexity of verifying whether a coalition of agents has full contribution to a temporal formula in a certain state.
\begin{proposition}
The model-checking problem for verifying the full contribution of a coaltion of agents can be solved in time $\mathcal{O}(m^2 \cdot n^2 \cdot l \cdot 2^n)$ for a self-organizing multi-agent system with $m$ states, a coalition of agents $|A|=n$ and a formula of length $l$.
\end{proposition}
\begin{proof}
Given a SOMAS $(\mathcal{S},\Gamma)$, a state $q$, a set of agents $A (|A|=n)$ and a temporal formula $\psi$ with the length of $l$, we need to follow Definition \ref{fullcontribution} to verify whether coalition $A$ has full contribution to $\psi$ in $q$. Since we know that the model-checking problem for ATL-$\Gamma$ can be solved in time $\mathcal{O}(m^2 \cdot n \cdot l)$, and that the model-checking problem for structural independence can be solved in time $\mathcal{O}(m^2 \cdot n^2)$, checking semantic and structural independence can be solved in time $\mathcal{O}(m^2 \cdot n^2 \cdot l)$. Moreover, we need to ensure that $A$ is the minimal coalition that is both semantically independent and structurally independent. Hence, we have to check every subset of $A$ for its semantic independence and structural independence. Therefore, checking the full contribution of a coalition of agents in a state is in the complexity of $\mathcal{O}(m^2 \cdot n^2 \cdot l \cdot 2^n)$.
\end{proof}

\section{Decomposing a Self-organizing Multi-agent System: a Layered Approach} \label{decomposition}
In Section \ref{abstractframework} we explored the dependence relation between agents from both structural and semantic perspectives. However, we still have no idea how agents guided by their respective local rules bring about the global behavior of the system. If we check formula $\mathcal{S},\Gamma,q \models \langle \Sigma \rangle \psi$ and it returns true, it is just proved that all the agents in the system have full contribution to the global behavior of the system $\psi$, which does not explain anything. For sure, we can enumerate all the possible coalitions of agents to check their contributions to the global behavior of the system, but it might be computationally expensive if the system has a large number of agents. Inspired by the decomposition approach in argumentation \cite{liao2013efficient}, we propose to decompose the system into different coalitions based on their dependence relation such that there is a partial order among different coalitions. Since a directed graph with nodes and arrows can better represent the dependence relation, we define a notion of a \emph{dependence graph} w.r.t. a set of computations. Formally, 
\begin{definition}[Dependence Graph]\label{graph}
Given an SOMAS $(\mathcal{S},\Gamma)$, a dependence graph w.r.t. a set of computations $out$ is a directed graph $G(V, E)$ where $V=\Sigma$ and $E=\{ (a,b) \mid$ there exists $\lambda \in out$ and $q \in \lambda$ such that $ a \in \tau_b(q)\}$. Typically, given a state $q$ and a coalition of agents $A$, when $G$ is a dependence graph w.r.t. a set of computations $out(q,\Gamma_A)$, we will denote it as $G(q,\Gamma_A)$.
\end{definition}
In words, a dependence graph w.r.t. a set of computations consists of a set of nodes $V$, which are the agents in $\mathcal{S}$, and a set of arrows $E$, each of which indicates one agent gets input from another agent in a state along a computation in $out$. Having this graph allows us to analyze the dependence relation between agents and the dependence relation between coalitions of agents with respect to a set of computations.
\begin{proposition}
Given an SOMAS $(\mathcal{S},\Gamma)$, a coalition $A$ is an independent component w.r.t. a set of computations $out$ iff $A$ does not get input from agents outside of $A$ in the corresponding dependence graph $G$.
\end{proposition}
\begin{proof}
Coalition $A$ is an independent component w.r.t. a set of computations $out$ iff for all $\lambda \in out$ and $q \in \lambda$ $A$ is an independent component w.r.t. state $q$ iff for all $\lambda \in out$ and $q \in \lambda$ and $a \in A$ it is the case that $\tau_a(q) \subseteq A$, which means that for all $\lambda \in out$ and $q \in \lambda$ and $a \in A$ it is the case that it is the case that $a$ does not get input from agents in $\Sigma \backslash A$, iff any $a \in A$ does not get input from agents outside of $A$ in the corresponding dependence graph $G$.
\end{proof}
From that, we can check whether a coalition of agents is an independent component w.r.t. $out$ through simply checking its corresponding dependence graph. 
\begin{proposition}\label{subsetcomput}
Given an SOMAS $(\mathcal{S},\Gamma)$ and two sets of computations $out$, if a coalition of agents $A$ is an independent component w.r.t. a set of computations $out$, then $A$ is also an independent component w.r.t. any set of computations $out^\prime$ where $out^\prime \subseteq out$.
\end{proposition}
\begin{proof}
Let $G$ be the dependence graph of $out$ and $G^\prime$ be the dependence graph of $out^\prime$. Because $out^\prime \subseteq out$, $G^\prime$ is a spanning subgraph of $G$. Therefore, if $A$ does not get input from other agents in $G$, $A$ will also not get input from other agents in $G^\prime$, which means that if $A$ is an independent component w.r.t. $out$, then $A$ is also an independent component w.r.t. $out^\prime$.
\end{proof}
Typically, given an SOMAS $(\mathcal{S},\Gamma)$, if a coalition of agents $A$ is an independent component w.r.t. $out(q, \Gamma_A)$, then $A$ is also an independent component w.r.t. $\lambda^*(q)$; if $A$ is not an independent component w.r.t. $\lambda^*(q)$, then $A$ is also not an independent component w.r.t. $out(q, \Gamma_A)$. In this section, we will use a more complicated example illustrate our definitions.
\begin{example}
A multi-agent system consists of a couple of agents, each of which has specific capacity that is not common knowledge among agents. A complicated task is delegated to the agents and the agents have to cooperate to finish the task in a self-organizing way. The local rule for each agent is as follows: each agent works on one part of the task based on its capacity; once its part is finished, it passes the rest of the task to other agents who can continue via wireless signals until the whole task is finished. We will use our framework to study the contributions of agents and how they cooperate to finish the whole task. Given agents' capacities, the whole task is finished along the computation $\lambda^*(q)$. The dependence graph $G$ w.r.t. $\lambda^*(q)$ is represented as Figure.\ref{figure_3}(left), which consists of 5 agents $\{a,b,c,d,e\}$ and communication between them regarding the task information. 
\begin{figure}
\centering
\includegraphics[width=0.8\textwidth]{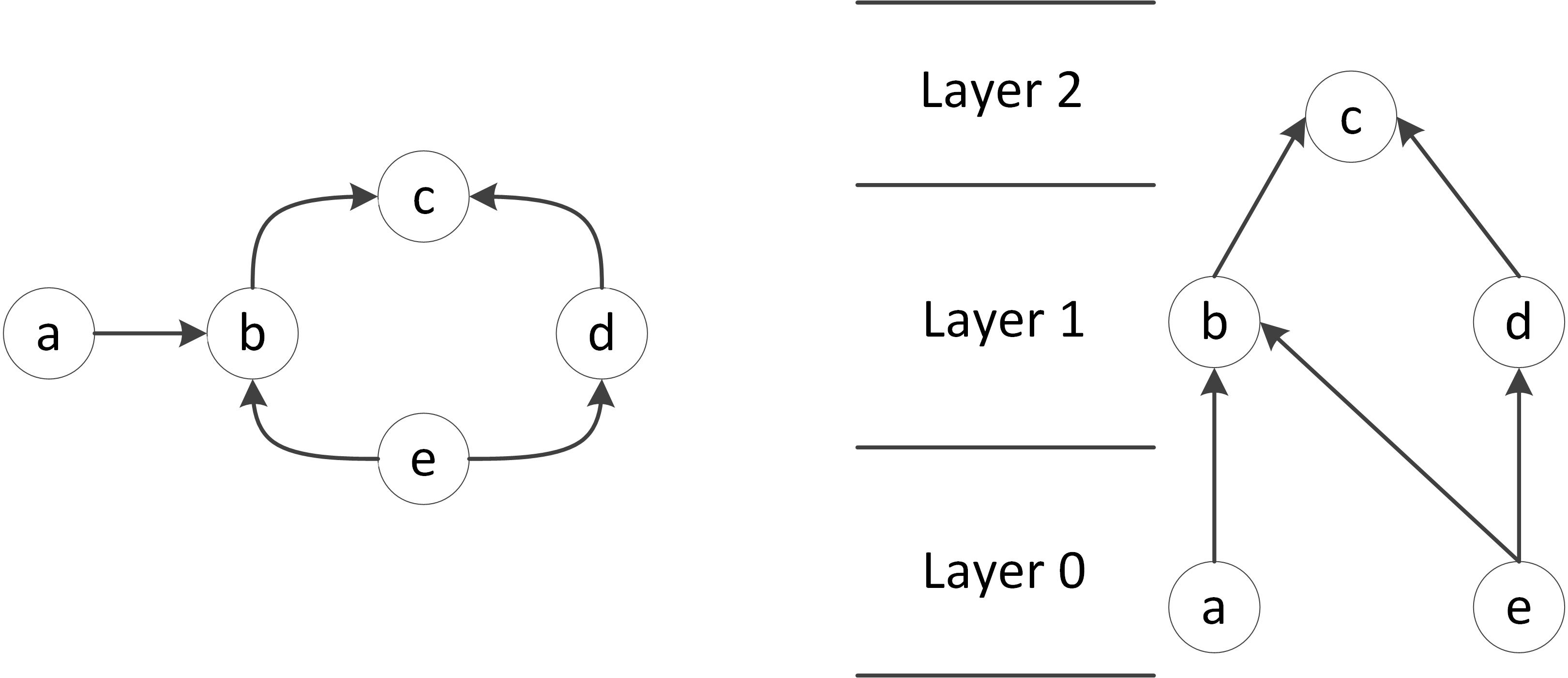}
\caption{A dependence graph (left) and a layered decomposition (right).}\label{figure_3}
\end{figure}
\end{example}

A dependence graph clearly illustrates not only the dependence relation between agents. When there exists circles in a dependence graph, we can shrink them down to single nodes based on the theory about strongly connected components and condensation graphs in graph theory. Because it is not the main concern of this paper, we omit the process of transforming a dependence graph to a condensation graph and assume that the dependence graph under consideration is a directed acyclic graph. As we mentioned before, we want to decompose the system such that there is a partial order of dependence among coalitions of agents. How to decompose the system based on the dependence graph becomes a problem we need to solve now. Similar to the solution in \cite{liao2013efficient}, we use a layering decomposition approach in this paper. We first propose the notion of \emph{layer}. Formally, 
\begin{definition}[Layer of an Agent]\label{layer}
Given a dependence graph $G(V,E)$ w.r.t. a set of computations $out$, for all $a \in V$, the layer of $a$ is a function $\rho: V \to \mathbb{N}$, and $\rho(a)$ is defined as:
\begin{itemize}
\item if $a$ has no parent in $G$, then $\rho(a)=0$;
\item otherwise, $\rho(a)=\max\{\rho(p)+1: p \text{ is a parent of } a\}$,
\end{itemize}
We use $h=\max\{\rho(a): a \in V\}$ to denote the highest layer.
\end{definition}
The above definition indicates the layer of an agent in a given dependence graph in two cases. If $a$ is a leaf node in the dependence graph $G$, then $a$ is located in the lowest layer such that it is independent on the behavior of any other agents. If $a$ is not a leaf node, then $a$ is located in the layer above all of its parents. A self-organizing multi-agent system can be decomposed into a number of layers. Formally,
\begin{definition}[Decomposition of an SOMAS]\label{decomposition}
Given an SOMAS $(\mathcal{S},\Gamma)$, a decomposition of $(\mathcal{S}, \Gamma)$ w.r.t. a set of computations $out$, denoted as $\decomp{\mathcal{S}, \Gamma, out}$, is a tuple 
\[\decomp{\mathcal{S}, \Gamma, out}=(L_0, L_1, \cdots, L_{h}),\] 
where $L_i=\{a \in \Sigma \mid \rho(a)=i\}$ ($0 \leq i \leq h$).
\end{definition} 
Using this decomposition approach, any agents in the same layer are independent on each other, and any agents in a given layer are only dependent on the agents located in the lower layers. It is expressed by the following proposition:
\begin{proposition}\label{prop:independence}
Given an SOMAS $(\mathcal{S}, \Gamma)$ and a decomposition $\decomp{\mathcal{S}, \Gamma, out}$, for any $a \in L_i, b \in L_j$ ($0 \leq i,j \leq h$):
\begin{itemize}
\item if $i = j$, then $a$ and $b$ do not get input from each other;
\item if $b$ gets input from $a$, then $i < j$.
\end{itemize}
\end{proposition}
\begin{proof}
If $i=j$, agents $a$ and $b$ are on the same level. For $a$ located on $L_i$, if $i=0$, it has no parent, thus it does not get input from $b$; if $i>0$, the parents of $a$ are located below $L_i$, thus $b$ cannot reach $a$, meaning that $a$ does not get input from $b$. Similarly, $b$ does not get input from $a$. If $b$ gets input from $a$, then there exists an arrow coming from $a$ to $b$. According to our decomposition approach, $a$ is located below $b$. Hence, $i < j$.
\end{proof}
Thus, we organize the system in a way where there is a partial order of dependence among agents located in different layers. This layered structure helps us to find coalitions of agents which do not get input from other agents from the dependence graph efficiently. Starting from $L_0$, any coalitions located in $L_0$ do not get input from other agents. We then go to $L_1$, where each agent combining with its parents located in lower layers does not get input from other agents. After that, we go to a higher level until we reach $L_h$. We have the following theorem to check whether a coalition of agents has full contribution to a temporal formula in a state.
\begin{theorem}\label{prop:theorem}
Given an SOMAS $(\mathcal{S}, \Gamma)$, a coalition of agents $A$ fully contributes to a temporal formula $\psi$ in a state $q$ iff $A$ satisfies the following conditions
\begin{itemize}
\item $A$ does not get input from agents in $\Sigma \backslash A$ in the dependence graph $G(q,\Gamma_A)$;
\item $\mathcal{S}, \Gamma, q \models \langle A \rangle \psi$;
\item for any $A^\prime \subset A$ if it does not get input from agents in $\Sigma \backslash A^\prime$ in $G(q,\Gamma_{A^\prime})$ then $\mathcal{S}, \Gamma, q \not \models \langle A^\prime \rangle \psi$.
\end{itemize}
\end{theorem}
\begin{proof}
For the first condition, coalition $A$ does not get input from agents in $\Sigma \backslash A$ in the dependence graph $G(q,\Gamma_A)$ if and only if coalition $A$ is an independent component w.r.t. $out(q,\Gamma_A)$, which means the coalition $A$ is structurally independent. The second condition ensures that $A$ is semantically independent w.r.t. $\psi$. For the third condition, for any subsets of $A$ $A^\prime \subset A$, if it gets input from agents in $\Sigma \backslash A^\prime$ in $G(q,\Gamma_{A^\prime})$, it is not structurally independent; if it does not get input from agents in $\Sigma \backslash A^\prime$ in $G(q,\Gamma_{A^\prime})$ and $\mathcal{S}, \Gamma, q \not \models \langle A^\prime \rangle \psi$, it is structurally independent but not semantically independent. Thus, for any $A^\prime \subset A$, it is either not structurally independent or not semantically independent, which means that $A$ is the minimal (w.r.t. set-inclusion) coalition that is both semantically independent with respect to $\psi$ and structurally independent from $q$. Therefore, $A$ fully contributes to a temporal formula $\psi$ in a state $q$.
\end{proof}
It indicates how to verify whether a coalition of agents $A$ has full contribution to $\psi$ in state $q$: structurally, we need to check whether $A$ does not get input from agents outside $A$ in the dependence graph of $out(q,\Gamma_A)$; semantically, we need to check whether $A$ ensures $\psi$ by following local rules; for any of its proper subsets $A^\prime$, if $A^\prime$ has been confirmed to be structurally independent w.r.t. $out(q,\Gamma_{A^\prime})$, then we need to check whether it is not semantically independent w.r.t. $\psi$. In other words, for any proper subsets of $A$, we need to ensure it is either not structurally independent or not semantically independent w.r.t. $\psi$ from state $q$. 

Given a finite set of temporal formulas tformulas that we would like to verify, we need to follow a procedure to efficiently examine the contribution of agents to the global system behavior rather than checking an arbitrary coalition. The results from Proposition \ref{prop:independence} and Theorem \ref{prop:theorem} can be used to design such a procedure. According to Proposition \ref{subsetcomput} if a coalition of agents $A$ is not an independent component w.r.t. $\lambda^*(q)$ then it is also not an independent component w.r.t. $out(q,\Gamma_A)$. Thus, we investigate the dependence graph $G$ of $\lambda^*(q)$ so that we can rule out the coalitions that are not independent components in $G$. Given a computation $\lambda^*(q)$ and its corresponding dependence graph $G$, we first find out all the coalitions of agents that do not get input from other agents according to the result of Proposition \ref{prop:independence}, putting them in the order of set inclusion. From the first coalition to the last coalition in the order, we follow Theorem \ref{prop:theorem} to check whether it has full contribution to any temporal property in tformulas. In this part, after we confirm that a coalition of agents $A$ is structurally and semantically independent w.r.t. a temporal formula $\psi$, we check whether $A^\prime \subset A$, which does not get input from other agents in $G(q,\Gamma_{A^\prime})$, also bring about $\psi$. If no, then $A$ is the minimal coalition (w.r.t. set-inclusion) coalition that is both semantically independent with respect to $\psi$ and structurally independent from $q$. We provide the following pseudocode to illustrate the process.
\begin{algorithm}
\caption{Finding Agents'Full Contributions in an SOMAS}
\begin{algorithmic}
\Function{FConSOMAS}{$\mathcal{S},\Gamma,q,tformulas$}
\State def tformulas;
\State def hashmap = FindIndependentComponent($\lambda^*(q)$);

\For{each $A$ in hashmap}
	\If{$A$ is not an independent component w.r.t. $out(q,\Gamma_A)$}
		\State hashmap.remove(A);
	\EndIf
	\For{each $\psi$ in tformulas}
		\If{$\mathcal{S},\Gamma,q \models \psi$}
			\For{each $A^\prime \subset A$}
				\If{hashmap.get($A^\prime$)!= null \& $\mathcal{S},\Gamma,q \not \models \psi$}
					\State hashmap.put($A$, $\psi$);
				\EndIf
			\EndFor
		\EndIf
	\EndFor
\EndFor
\State \Return hashmap;
\EndFunction
\end{algorithmic} 
\end{algorithm}

When it is finished, we collect the information about agents' full contributions in a set, which allows us to understand how different coalitions of agents contribute to the global system behavior. 
\begin{definition}[Agents' Independence in Terms of Full Contributions]
Given an SOMAS $(\mathcal{S}, \Gamma)$ and a state $q$, agents' independence in terms of their full contributions in $(\mathcal{S}, \Gamma)$ is a set $F(q)$ containing the information of agents' full contributions, namely $F(q)=\{(A,\psi_A) \mid$ Coalition $A$ has full contribution to $\psi_A$ in state $q$. $\}$.
\end{definition}
\begin{example}
A decomposition of dependence graph $G$ in Fig.\ref{figure_3}(right) is as follows. According to Definition \ref{layer}, $\rho(a)=0$, $\rho(b)=1$, $\rho(c)=2$, $\rho(d)=1$ and $\rho(e)=0$. Therefore, the system is decomposed into three layers: $\decomp{\mathcal{S}, \Gamma}=(L_0, L_1,L_2)$, where $L_0=\{a,e\}$, $L_1=\{b,d\}$ and $L_2=\{c\}$. Based on the decomposition, from $L_0$ to $L_2$, we can find multiple independent components, namely $\{a\}$, $\{e\}$, $\{a,e\}$, $\{a,b,e\}$, $\{d,e\}$, $\{a,d,e\}$, $\{a,b,d,e\}$ and $\Sigma$. We then examine whether they fully contribute to any temporal formula one by one, which does a lot of savings over checking $2^5$ possible coalitions without using the dependence graph. Starting from $\{a\}$, we check whether agent $a$ also does not get input from other agents in the dependence graph $G(q, \Gamma_{\{a\}})$, and check whether it ensures $\psi_a$ by following its local rule. If both return yes, then $a$ fully contributes to $\psi_a$. We perform the same examination for coalition $\{e\}$. For coalition $\{a,e\}$, after confirming that it is both semantically independent with respect to $\psi$ and structurally independent from $q$, we still need to check whether neither coalitions $\{a\}$ nor $\{e\}$ brings about $\psi_{ae}$ to ensure the minimization requirement. The verification is finished when we have done with $\Sigma$. Suppose we have the following result: $F(q)=\{(\{a\}, \psi_a), (\{d,e\}, \psi_{de}), (\{a,b,e\}, \psi_{abe}), (\Sigma, \psi)\}$, each of which indicates that a coalition has full contribution to a part of the task. Therefore, $F(q)$ contains the information on how different coalitions of agents finish the whole task. 
\end{example}

\section{Modeling Constraint Satisfaction Problems}
In mult-agent systems, it is normal that agents have personal goals and constraints. From the perspective of the system, we would like to find a solution such that we can balance the satisfaction of the goals and constraints among agents. In mathematics, those problems are called constraint satisfaction problems. Constraint satisfaction problems are NP-complete in general, which means that it will more and more computationally expensive to solve those problems in centralized way when we enlarge the size of the problem. After realizing that, people start to look at agents themselves and the cooperation among them: agents can negotiate as to reach a global equilibrium. Algorithms based on self organization have been developed to solve various constraint satisfaction problems, such as task/resource allocation \cite{macarthur2011distributed} and relation adaption \cite{ye2012self}. Formally, a constraint satisfaction problem is defined as a triple $(X,D,C)$, where 
\begin{itemize}
\item $X=\{x_1,\ldots,x_n\}$ is a set of variables; 
\item $D=\{D_1,\ldots,D_n\}$ is a set of domains; 
\item $C=\{C_1,\ldots,C_m\}$ is a set of constraints. 
\end{itemize}
Each variable $x_i$ can take on the values in its nonempty domain $D_i$. Every constraint $C_j \in C$ is a pair $(t_j,R_j)$, where $t_j \subset X$ is a subset of n variables and $R_j$ is the relation on $t_j$. An evaluation of the variables is a function $v$ from a subset of variables to a particular set of values in the corresponding subset of domains. An evaluation $v$ is said to satisfy a constraint $(t_j,R_j)$ if the values assigned to the variables $t_j$ satisfies the relation $R_j$. A solution to the constraint satisfaction problem is an evaluation that include all the variables and do not violate any of the constraints. 

Using our framework of self-organizing multi-agent systems to model a constraint satisfaction problem, we can see a complete evaluation that includes all the variables as a particular state in the system, and a self-organization based algorithm is a set of local rules we design for agents to follow. The goal of the system is to converge to a stable state where the values of $X$ do not violate any constraints, which is a feasible solution to the problem. Notice that set $C$ consists of multiple constraints. An evaluation that satisfies constraints $C^\prime$ always implies an evaluation that satisfies constraints $C^{\prime\prime}$ if $C^{\prime\prime} \subseteq C^\prime$. Next, we will use a real example to illustrate how we use our framework to model a constraint satisfaction problem.

Multi-agent systems can be used to support information exchange within user communities where each user agent represents a user's interests. But how does a user agent make contact with other user agents with common interests? We can use a peer-to-peer network, but this could flood the network with queries, risking overloading the system. A decentralized solution based on middle agents can self-organize to form communities among agents with common interests \cite{wang2002self}. User agents represent users and each user agent registers with one or more middle agent. User agents (requestors) send queries about their user’s interests to middle agents. Given these queries, each middle agent then search among the information it holds from other user agents registered with it. If it can respond to a query based on this information then the search is completed. Otherwise, the middle agent communicates with other middle agents in order to try and obtain the information. Once this has been done, the middle agent relays results of the search to the requester. After that, the middle agent checks whether both the requester and the provider are within its group. If not, the middle agent of the requester transfers the provider to its group, so that both user agents are in the same group. The consequence is that users with common interests form a community. The querying behavior of user agents builds up a profile of their interests, which is updated by successive queries. It is a solution with self-organization features that solves a constraint satisfaction problem that the middle agents where user agents register correspond to the set of variables $X$, the possible middle agents where user agents can register correspond to the domains of the variables $D$, and user agents' interests in other user agents correspond to the constraints $C$. Compared to centralizedly solving the problem that might encounter computational issues when the number of users increases, the above solution is a highly scalable process that operates efficiently with a large number of users. As user agents broadcast their queries in the system when implementing the solution, it is not required that the system designer is aware of user agents' interests beforehand.

\begin{figure}
\centering
\includegraphics[width=0.9\textwidth]{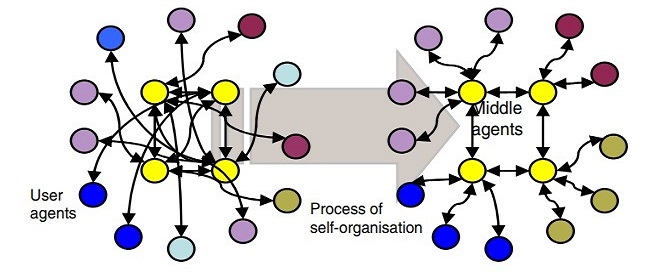}
\caption{Self organization of agent communities - different colours represent different user interests.}\label{figure_6}
\end{figure}

We now use our framework to model the example. The system has a set of user agents $U$ and a set of middle agents $M$; each middle agent can register $reg$, disregister $disreg$ a user agent or does nothing. Function $int:U \to \mathcal{P}(U)$ returns the static interest of each user agent $u \in U$. Each state $q$ is labeled with propositions, each of which is of the form $reg(u,m)$, meaning that user agent $u$ registers in middle agent $m$. We assume that middle agents are aware of the positions where user agents register, but not the interests of user agents before user agents broadcast their queries in the system. The system moves from one state to another state once a user agent registers or disregister with a user agent. The goal of the system is to converge to a state $q^*$ where all user agents are in their own user communities. As Fig. \ref{figure_6} described, a user community consists of user agent(s) and middle agent(s); for any user agent in the community, it registers with the same middle agent as its interested user agents do. In order to formalize this property, we extend our semantics as follows: given $U^\prime \subseteq U$ and $M^\prime \subseteq M$,
\begin{itemize}
\item $\mathcal{S},\Gamma,q \models \com{U^\prime}{M^\prime}$ iff for all $u \in U^\prime$ and $u^\prime \in int(u)$ there exists $m \in M^\prime$ such that $\mathcal{S},\Gamma,q \models reg(u,m)\land reg(u^\prime,m)$.
\end{itemize}

Local rule $\langle \tau, \gamma \rangle$ for each agent is prescribed as follows: if a user agent $u_1$ has interest in another user agent $u_2$ while they don't register with the same middle agent, then the middle agent of $u_2$  has to disregister $u_2$ and send the information about $u_2$ to the middle agent of $u_1$ so that it can register $u_2$. Suppose at the starting state $q_0$ user agents randomly register with different middle agents and each has interest in some other user agents (see Table \ref{table1}). Starting from $q_0$, $u_1$ asks a question about $u_2$. After hearing the query about $u_2$, $m_2$ disregisters $u_2$ and sends information about $u_2$ to $m_1$, which moves the system from state $q_0$ to state $q_1$. After getting the information about $u_2$ from $m_2$, $m_1$ registers $u_2$, which moves the system from state $q_1$ to state $q_2$. In state $q_2$, both $u_1$ and $u_2$ register with the same middle agent $m_1$. Later on, $u_2$ asks a question about $u_1$. Because $u_1$ and $u_2$ register with the same middle agent $m_1$, the system does not change anything from the query. And then $u_3$ asks a question about $u_4$. Similar to $u_2$, $m_1$ disregisters $u_4$ after hearing the query, and $m_3$ registers $u_4$ after getting the information about $u_4$ from $m_1$. And then $u_4$ asks a question about $u_3$, which does not trigger any changes to the system. Based on the local rules, the dependence graph between user agents and middle agents in terms of their communication along the computation $\lambda^*(q_0)$ is shown in Fig.~\ref{figure_7}. Triggered by the queries, all agents follow their local rules to behave and the system transits from $q_0$ to $q_4$. See Fig.~\ref{figure_8} for the state transitions of the system and the valuation of each state. Because in $q_4$ both $u_1$ and $u_2$ with common interests register with $m_1$ and both $u_3$ and $u_4$ with common interests register with $m_3$, $q_2$ satisfies the property of our desired state. Starting from a disorder state $q_0$, the system finally converges to state $q_4$ where for each user $u \in U$ it registers with the same middle agent as its interested user agents do, forming two user communities $\com{\{u_1,u_2\}}{\{m_1\}}$ and $\com{\{u_3,u_4\}}{\{m_3\}}$. How does the system reach this desired state? What are the agents' contributions to forming such two communities? Here is the analysis.

\begin{table}[htbp]
\centering 
\caption{User agents' interests and their initial registration positions}
\label{table1}
\begin{tabular}{|c|c|c|} 
\hline 
& & \\[-6pt] 
User Agent&Interest&Register in \\
\hline
& & \\[-6pt] 
$u_1$&$u_2$&$m_1$ \\
\hline
& & \\[-6pt] 
$u_2$&$u_1$&$m_2$ \\
\hline
& & \\[-6pt] 
$u_3$&$u_4$&$m_3$ \\
\hline
& & \\[-6pt] 
$u_4$&$u_3$&$m_1$ \\
\hline
\end{tabular}
\end{table}

\begin{figure}
\centering
\includegraphics[width=0.6\textwidth]{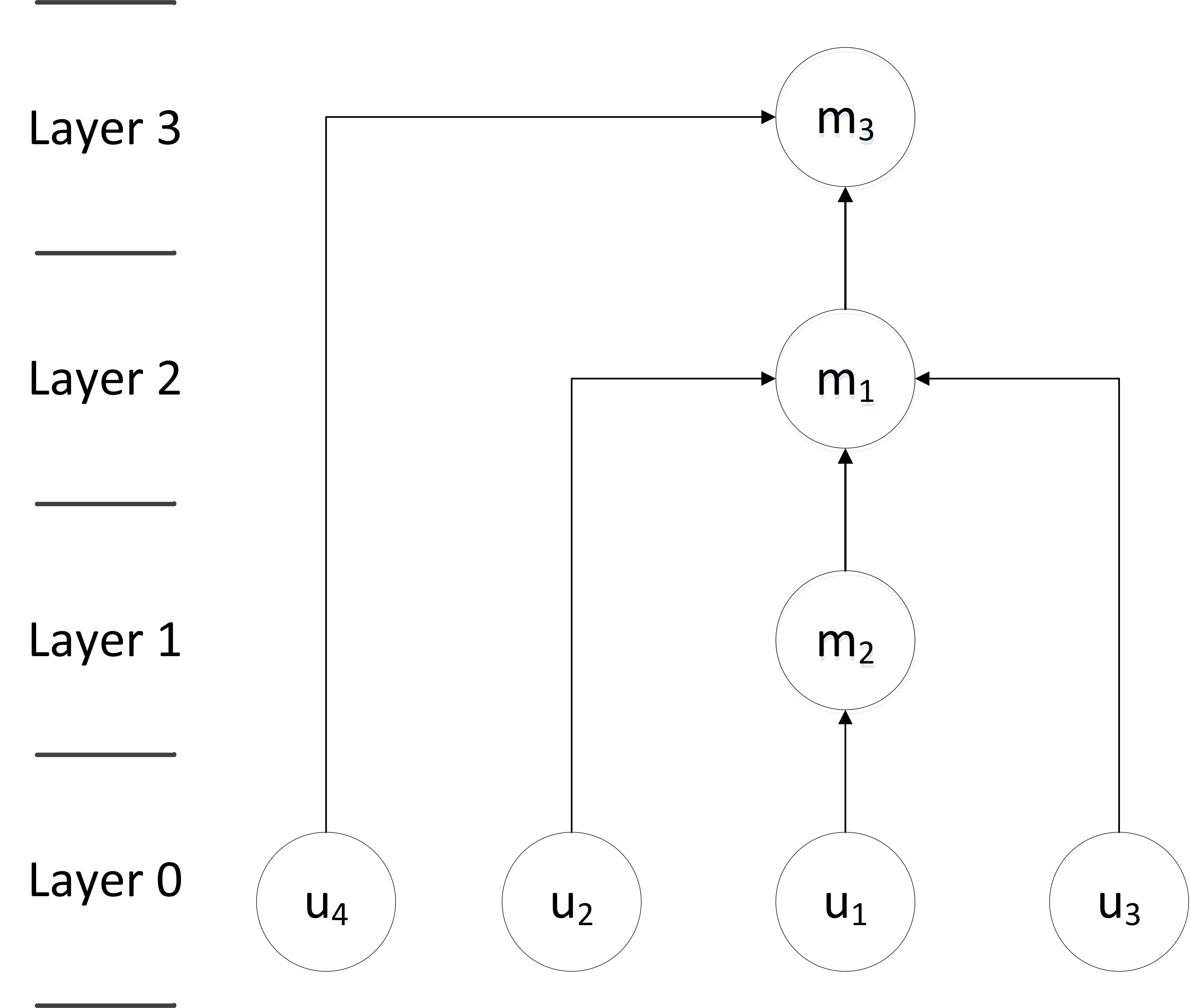}
\caption{Dependence graph between user agents and middle agents along the computation $\lambda^*(q_0)$, where all agents follow their local rules to behave.}\label{figure_7}
\end{figure}

\begin{figure}
\centering
\includegraphics[width=0.9\textwidth]{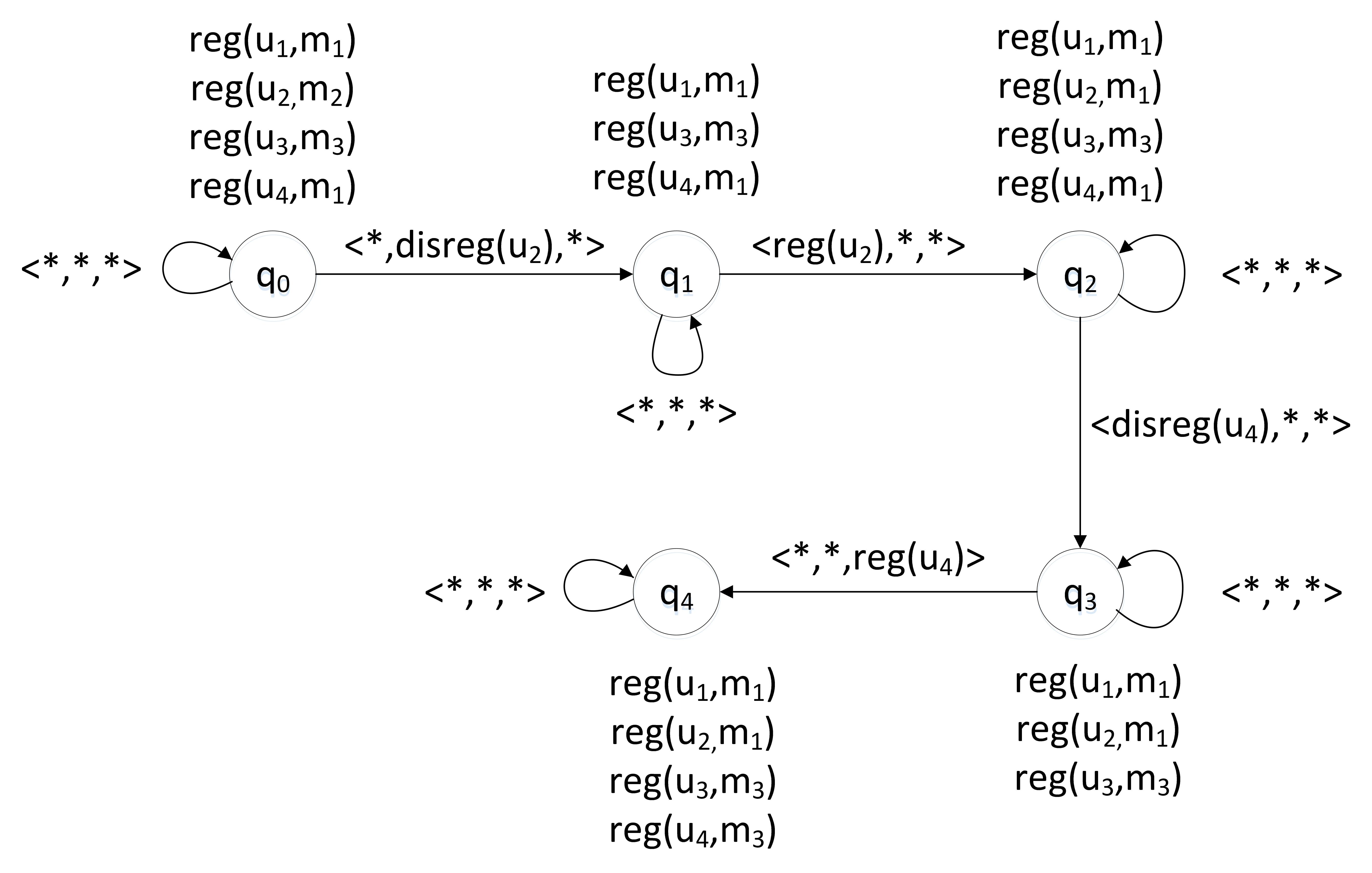}
\caption{State transitions of the system.}\label{figure_8}
\end{figure}

From Fig. \ref{figure_7}, we can decompose the system into four layers, which are all the user agens as layer 0, $\{m_2\}$ as layer 1, $\{m_1\}$ as layer 2 and $\{m_3\}$ as layer 3. We can see there are lots of independent components w.r.t. computation $\lambda^*(q_0)$. Because the contributions of user agents are not encoded in the system, we do not discuss the independent components that are formed only by user agents. Given a set of temporal formulas $\{\Diamond \com{\{u_1,u_2\}}{\{m_1\}}, \Diamond \com{\{u_3,u_4\}}{\{m_3\}}\}$, we verify which coalition of agents full contributes to any of them. Firstly, coalition $\{u_1,u_2,m_1,m_2\}$ has full contribution to $\Diamond \com{\{u_1,u_2\}}{\{m_1\}}$. This is because any agents in the coalition do not get input from agents outside the coalition, and they form the community of $\{u_1,u_2,m_1\}$, which can be expressed by
\[\mathcal{S},\Gamma,q_0 \models \langle u_1, u_2, m_1,m_2\rangle \Diamond \com{\{u_1,u_2\}}{\{m_1\}},\]
and apparently it cannot be achieved without any of them, which satisfies the minimization condition. Secondly, consider community $\com{\{u_3,u_4\}}{\{m_3\}}$. From Fig. \ref{figure_7}, we can see that coalition $\{u_3,u_4,m_3\}$ is not an independent component as it gets input from middle agent $m_1$. Finally, the whole system fully contributes to forming the two communities $\com{\{u_1,u_2\}}{\{m_1\}}$ and $\com{\{u_3,u_4\}}{\{m_3\}}\}$. 

Such information regarding agents' full contributions allows us to understand how user agents' interest are satisfied and how user communities are formed, which facilitates the development of the system. For example, since coalition $\{u_1,u_2,m_1,m_2\}$ has full contribution to $\Diamond \com{\{u_1,u_2\}}{\{m_1\}}$, the community will not change even if we change the interests of $u_3$ and $u_4$ or $m_3$ does not register $u_4$.

\section{Related Work}
In order to understand the complex link between local agent behavior and system level behavior, we need to study the \emph{independence} between local agent behavior in terms of ensuring properties. Coalition-related logic achieves that from different perspectives. For instance, alternating temporal logic (ATL) and game theory can be used to check whether a coalition of agents can enforce a state property regardless of what other agents do \cite{Knobbout2012Reasoning}. \cite{aagotnes2010robust} study whether a normative system remains effective even though some agents do not comply with it. \cite{wu2011framework} identifies two types of system properties that are unchangeable by restricting the joint behavior of a coalition. Apart from the semantics dimension, we have proposed the notion of independent components, a coalition of agents which only get input information from agents in the coalition. It is a structural dimension to represent a coalition of agents whose behavior is independent on the behavior of agents outside of the coalition.

If we see local rules as norms that are used to regulate the behavior of the systems, self-organizing multi-agent systems are actually a category of normative multi-agent systems. In preventative control systems \cite{shoham1993agent}, norms are represented as hard constraints where violations are impossible, but this does not respect agents' autonomy. Soft constraints are used in detective control systems where norms are possible to be violated but agents are motivated to follow the norms through sanctions or punishments \cite{brafman1996partially}. It is indeed more flexible than setting hard constraints in the system. However, agents coming from an open environment have their own personal situations such as knowledge and preferences, which might not be known to the designer of the system. In such a case, the system designer cannot identity which outcome is desired thus it is hard to identify legal computations to get there \cite{bulling2016norm}\cite{knobbout2016formal}. The norms that we consider in this paper are designed based on agents' communication. As the norm prescribes, agents are supposed to communicate with each other about their own situations and what agents need to do depends on the communication results. Such type of norms are in fact widely used in multi-agent systems because it allows agents to regulate the system by themselves \cite{wang2002self}\cite{valentini2014self}\cite{macarthur2011distributed}.    

The logic we use in this paper is inspired by ATL \cite{alur2002alternating}. ATL is usually used to reason about the strategies of participating agents \cite{bulling2010modelling}\cite{goranko2018game}: formula $\llangle A \rrangle \psi$ is read as coalition $A$ have a strategy or can collaborate to bring about formula $\psi$, which can be seen as the capacity of the participating agents to ensure a result no matter what other agents outside of the coalition do. A strategy is a function that takes a state (imperfect recall) or a sequence of states (perfect recall) as an argument and outputs an action, which can be seen as a plan. In order to make a plan for a coalition of agents to ensure a property, there is an external that is aware of the available actions from each agent and the game structure. In [Alur et al., 2002], the authors use a game between a protagonist and an antagonist to gain more intuition about ATL formulas. The ATL formula $\llangle A \rrangle \psi$ is satisfied at state $q$ iff the protagonist has a winning strategy to achieve $\psi$ in this game, which means that the protagonist is aware of the available actions from each agent and the game so that he can make a strategy (or a plan) for coalition $A$ to win the game. However, in a self-organizing multi-agent system, agents only have local view of the system, which is in this paper the state where they are, their own internal functions and the actions that are available to them and the local rules they need to follow, but not the internal functions of other agents and the actions that are available to other agents. Hence, there is no external to centrally make a plan for agents to follow. The logic in this paper is used to reason about the local rules of the participating agents: formula $\langle A \rangle \psi$ is read as coalition $A$ bring about formula $\psi$ by following their local rules no matter what other agents outside the coalition do. The use of logic and graph theory is new in the research of multi-agent systems, but it has appeared in the area of argumentation. An abstract argumentation framework can be represented as a directed graph (called a defeat graph) where nodes represent arguments and arcs represent attacks between them \cite{dung1995acceptability}.  \cite{liao2013efficient} use strongly connected components in graph theory to decompose an abstract argumentation framework for efficient computation of argumentation extensions. The decomposition approach in this paper is inspired by that, while we use it not for efficient computation but for system verification.

While we use modal logic to understand the complex link between local agent behavior and system level behavior, one could think that causal reasoning is a plausible alternative. Causal reasoning is the process of identifying the relationship between a cause and its effect \cite{pearl2009causality}. It indeed has been used to identify the causal relation between macro- and micro-variables in \cite{chalupka2016unsupervised} \cite{chalupka2017automated}. For self-organizing multi-agent systems, the local rules that agents follow are the cause and the global property they contribute is the effect. But sometimes it might be difficult to identify the causal relation because of the interactions between agents. In such cases, causal reasoning can still be combined with our graph-based layered approach to decompose the system, which is similar to the idea proposed by Judea Pearl and Thomas Verma to combine logic and directed graphs for causal reasoning \cite{pearl1987logic}.

\section{Conclusion}
Self-organization has been introduced to multi-agent systems as an internal control process or mechanism to solve difficult problems spontaneously. However, because a self-organizing multi-agent system has autonomous agents and local interactions between them, it is difficult to predict the global behavior of the system from the behavior of the agents we design deductively, making implementation the usual way to evaluate a self-organizing multi-agent system. Therefore, we believe that if we can understand how agents bring about the behavior of the system in the sense that which coalition contributes to which system property independently, the development of self-organizing multi-agent systems will be facilitated. In this paper, we propose a logic-based framework for self-organizing multi-agent systems, where abstract local rules are defined in the way that agents make their next moves based on their communication with other agents. Such a framework allows us to verify properties of the system without doing implementation. A structural property called independent components is introduced to represent a coalition of agents which do not get input from agents outside the coalition. The dependence relation between coalitions of agents regarding their contributions to the global behavior of the system is reasoned about from the structural and semantic perspectives. We then show model-checking a formula in our language remains close to the complexity of model-checking standard ATL, while model-checking whether a coalition of agents has full contribution to a temporal property is in exponential time. Moreover, we combine our framework with graph theory to decompose a system into different coalitions located in different layers. The resulting information about agents' full contributions allows us to understand the complex link between local agent behavior and system level behavior in a self-organizing multi-agent system. We finally show how we can use our framework to model a constraint satisfaction problem, where a solution based on self-organization is used. Certainly, there may be some possible limitations in this study: for example, we only consider communication as the interaction between agents, which does not capture all types of interaction in the system. Moreover, the dependence graph with respect to computation $\lambda^*(q)$ is determined by agents' communication required by local rules. Thus, the internal function $m_i$ of each agent play an important role in the global system behavior. In the future, it will be interesting to investigate the robustness of self-organizing multi-agent systems due to the change of internal functions, which is possible to happen when the system is deployed in an open environment and thus agents can join or exit the system as they want.

\bibliography{sigproc}

\begin{thebibliography}{10}

\bibitem{aagotnes2010robust}
Thomas {\AA}gotnes, Wiebe Van~der Hoek, and Michael Wooldridge, `Robust
  normative systems and a logic of norm compliance', {\em Logic Journal of
  IGPL}, {\bf 18}(1),  4--30, (2010).

\bibitem{alur2002alternating}
Rajeev Alur, Thomas~A Henzinger, and Orna Kupferman, `Alternating-time temporal
  logic', {\em Journal of the ACM (JACM)}, {\bf 49}(5),  672--713, (2002).

\bibitem{bernon2003tools}
Carole Bernon, Val{\'e}rie Camps, Marie-Pierre Gleizes, and Gauthier Picard,
  `Tools for self-organizing applications engineering', in {\em International
  Workshop on Engineering Self-Organising Applications}, pp. 283--298.
  Springer, (2003).

\bibitem{bernon2002adelfe}
Carole Bernon, Marie-Pierre Gleizes, Sylvain Peyruqueou, and Gauthier Picard,
  `Adelfe: a methodology for adaptive multi-agent systems engineering', in {\em
  International Workshop on Engineering Societies in the Agents World}, pp.
  156--169. Springer, (2002).

\bibitem{brafman1996partially}
Ronen~I Brafman and Moshe Tennenholtz, `On partially controlled multi-agent
  systems', {\em Journal of Artificial Intelligence Research}, {\bf 4},
  477--507, (1996).

\bibitem{bulling2010modelling}
Nils Bulling, {\em Modelling and Verifying Abilities of Rational Agents},
  Papierflieger-Verlag, 2010.

\bibitem{bulling2016norm}
Nils Bulling and Mehdi Dastani, `Norm-based mechanism design', {\em Artificial
  Intelligence}, {\bf 239},  97--142, (2016).

\bibitem{chalupka2017automated}
Krzysztof Chalupka, {\em Automated Macro-scale Causal Hypothesis Formation
  Based on Micro-scale Observation}, Ph.D.\ dissertation, California Institute
  of Technology, 2017.

\bibitem{chalupka2016unsupervised}
Krzysztof Chalupka, Tobias Bischoff, Pietro Perona, and Frederick Eberhardt,
  `Unsupervised discovery of el nino using causal feature learning on
  microlevel climate data', in {\em Proceedings of the Thirty-Second Conference
  on Uncertainty in Artificial Intelligence}, pp. 72--81, (2016).

\bibitem{clarke1986automatic}
Edmund~M. Clarke, E~Allen Emerson, and A~Prasad Sistla, `Automatic verification
  of finite-state concurrent systems using temporal logic specifications', {\em
  ACM Transactions on Programming Languages and Systems (TOPLAS)}, {\bf 8}(2),
  244--263, (1986).

\bibitem{clarke2018model}
Edmund~M Clarke~Jr, Orna Grumberg, Daniel Kroening, Doron Peled, and Helmut
  Veith, {\em Model checking}, MIT press, 2018.

\bibitem{di2005self}
Giovanna Di~Marzo~Serugendo, Marie-Pierre Gleizes, and Anthony Karageorgos,
  `Self-organization in multi-agent systems', {\em Knowledge Engineering
  Review}, {\bf 20}(2),  165--189, (2005).

\bibitem{dorigo2006ant}
Marco Dorigo, Mauro Birattari, and Thomas Stutzle, `Ant colony optimization',
  {\em IEEE computational intelligence magazine}, {\bf 1}(4),  28--39, (2006).

\bibitem{dung1995acceptability}
Phan~Minh Dung, `On the acceptability of arguments and its fundamental role in
  nonmonotonic reasoning, logic programming and n-person games', {\em
  Artificial intelligence}, {\bf 77}(2),  321--357, (1995).

\bibitem{goranko2018game}
Valentin Goranko, Antti Kuusisto, and Raine R{\"o}nnholm, `Game-theoretic
  semantics for alternating-time temporal logic', {\em ACM Transactions on
  Computational Logic (TOCL)}, {\bf 19}(3),  1--38, (2018).

\bibitem{gorodetskii2012selfI}
VI~Gorodetskii, `Self-organization and multiagent systems: I. models of
  multiagent self-organization', {\em Journal of Computer and Systems Sciences
  International}, {\bf 51}(2),  256--281, (2012).

\bibitem{gorodetskii2012self}
VI~Gorodetskii, `Self-organization and multiagent systems: Ii. applications and
  the development technology', {\em Journal of Computer and Systems Sciences
  International}, {\bf 51}(3),  391--409, (2012).

\bibitem{khelil2016esa}
Abdelkader Khelil and Rachid Beghdad, `Esa: an efficient self-deployment
  algorithm for coverage in wireless sensor networks', {\em Procedia Computer
  Science}, {\bf 98},  40--47, (2016).

\bibitem{Knobbout2012Reasoning}
Max Knobbout and Mehdi Dastani, `Reasoning under compliance assumptions in
  normative multiagent systems', in {\em Proceedings of the 11th International
  Conference on Autonomous Agents and Multiagent Systems-Volume 1}, pp.
  331--340. International Foundation for Autonomous Agents and Multiagent
  Systems, (2012).

\bibitem{knobbout2016formal}
Max Knobbout, Mehdi Dastani, and John-Jules~Ch Meyer, `Formal frameworks for
  verifying normative multi-agent systems', in {\em Theory and Practice of
  Formal Methods},  294--308, Springer, (2016).

\bibitem{liao2013efficient}
B.~Liao, {\em Efficient Computation of Argumentation Semantics}, Intelligent
  systems series, Elsevier Science, 2013.

\bibitem{macarthur2011distributed}
Kathryn~Sarah Macarthur, Ruben Stranders, Sarvapali Ramchurn, and Nicholas
  Jennings, `A distributed anytime algorithm for dynamic task allocation in
  multi-agent systems', in {\em Twenty-Fifth AAAI Conference on Artificial
  Intelligence}, (2011).

\bibitem{pearl2009causality}
Judea Pearl, {\em Causality}, Cambridge university press, 2009.

\bibitem{pearl1987logic}
Judea Pearl and Thomas Verma, {\em The logic of representing dependencies by
  directed graphs}, University of California (Los Angeles). Computer Science
  Department, 1987.

\bibitem{picard2005etto}
Gauthier Picard, Carole Bernon, and Marie-Pierre Gleizes, `Etto: emergent
  timetabling by cooperative self-organization', in {\em International Workshop
  on Engineering Self-Organising Applications}, pp. 31--45. Springer, (2005).

\bibitem{shoham1993agent}
Yoav Shoham, `Agent-oriented programming', {\em Artificial intelligence}, {\bf
  60}(1),  51--92, (1993).

\bibitem{valentini2014self}
Gabriele Valentini, Heiko Hamann, and Marco Dorigo, `Self-organized collective
  decision making: The weighted voter model', in {\em Proceedings of the 2014
  international conference on Autonomous agents and multi-agent systems}, pp.
  45--52. International Foundation for Autonomous Agents and Multiagent
  Systems, (2014).

\bibitem{wang2002self}
Fang Wang, `Self-organising communities formed by middle agents', in {\em
  Proceedings of the first international joint conference on Autonomous agents
  and multiagent systems: part 3}, pp. 1333--1339. ACM, (2002).

\bibitem{wooldridge2005obligations}
Michael Wooldridge and Wiebe Van Der~Hoek, `On obligations and normative
  ability: Towards a logical analysis of the social contract', {\em Journal of
  Applied Logic}, {\bf 3}(3-4),  396--420, (2005).

\bibitem{wu2011framework}
Jun Wu, Chongjun Wang, and Junyuan Xie, `A framework for coalitional normative
  systems', in {\em The 10th International Conference on Autonomous Agents and
  Multiagent Systems-Volume 1}, pp. 259--266, (2011).

\bibitem{ye2012self}
Dayong Ye, Minjie Zhang, and Danny Sutanto, `Self-organization in an agent
  network: A mechanism and a potential application', {\em Decision Support
  Systems}, {\bf 53}(3),  406--417, (2012).

\bibitem{ye2016survey}
Dayong Ye, Minjie Zhang, and Athanasios~V Vasilakos, `A survey of
  self-organization mechanisms in multiagent systems', {\em IEEE Transactions
  on Systems, Man, and Cybernetics: Systems}, {\bf 47}(3),  441--461, (2016).

\end{thebibliography}
\end{document}